\documentclass{article}
\PassOptionsToPackage{numbers,compress}{natbib}

\usepackage[preprint]{neurips_2026}

\usepackage[utf8]{inputenc} 
\usepackage[T1]{fontenc}    
\usepackage{hyperref}       
\usepackage{url}            
\usepackage{booktabs}       
\usepackage{amsfonts}       
\usepackage{nicefrac}       
\usepackage{microtype}      
\usepackage{xcolor}         
\usepackage{graphicx}
\usepackage{amsthm}
\usepackage{amsmath}
\usepackage{subcaption}
\usepackage{multirow} 
\usepackage{diagbox}
\usepackage{wrapfig}
\usepackage[table]{xcolor}
\usepackage{algorithm}
\usepackage{algorithmic} 
\newtheorem{definition}{Definition}
\newtheorem{theorem}{Theorem}
\usepackage{wrapfig}

\newtheorem{lemma}{Lemma}
\newtheorem{proposition}{Proposition}

\definecolor{rowgray}{RGB}{245,245,245}       
\definecolor{arrowgray}{RGB}{128,128,128}     
\definecolor{spiderblue}{RGB}{204, 247, 255}  
\definecolor{highlightyellow}{RGB}{247, 230, 169} 

\definecolor{finalrowgray}{RGB}{230, 230, 230} 

\newcommand{\BestData}[2]{\textbf{#1$_{\pm #2}$}}

\newcommand{\UnderlineDataFull}[2]{\underline{#1$_{\pm #2}$}}

\newcommand{\Data}[2]{#1$_{\pm #2}$}

\title{HeterSEED: Semantics–Structure Decoupling for Heterogeneous Graph Learning under Heterophily}

\author{%
  David S.~Hippocampus\thanks{Use footnote for providing further information
    about author (webpage, alternative address)---\emph{not} for acknowledging
    funding agencies.} \\
  Department of Computer Science\\
  Cranberry-Lemon University\\
  Pittsburgh, PA 15213 \\
  \texttt{hippo@cs.cranberry-lemon.edu} \\
}

\author{%
Xinyi Li$^{1}$, Ming Li$^{1}$, Lu Bai$^{2}$, Lixin Cui$^{3}$, Feilong Cao$^{1}$, Ke Lv$^{4}$, Yunliang Jiang$^{1}$, Pietro Liò$^{5}$\\[0.3em]
$^{1}$Zhejiang Normal University\\
$^{2}$Beijing Normal University\\
$^{3}$Central University of Finance and Economics\\
$^{4}$University of Chinese Academy of Sciences\\
$^{5}$University of Cambridge\\
}

\begin{document}

\maketitle

\begin{abstract}
Many real-world heterogeneous graphs exhibit pronounced heterophily, where connected nodes often have dissimilar labels or play different semantic roles. In such settings, standard heterogeneous graph neural networks that aggregate messages along metapaths or meta-relations primarily based on feature similarity can propagate misleading information, since feature similarity may be misaligned with underlying relational semantics. In this paper, we propose \textbf{HeterSEED}, a semantics–structure decoupling framework for heterogeneous graph learning under heterophily. HeterSEED decouples representation learning into a heterogeneous semantic channel that captures type- and relation-aware local semantics and a structure-aware heterophily channel that separates homophilic and heterophilic neighborhoods via pseudo-label-guided partitioning and aggregates them using metapath-based structural weights. A node-level adaptive fusion mechanism then combines the two channels to produce context-dependent node representations. Theoretically, we establish that, on heterogeneous graphs under heterophily, HeterSEED is strictly more expressive than standard heterogeneous graph neural networks that rely primarily on feature similarity and provably reduces the prediction bias introduced by heterophilic neighbors. Experiments on five real-world heterogeneous graphs, including two large-scale networks at the million-node and hundred-million-edge scale, demonstrate that HeterSEED consistently outperforms representative heterogeneous graph neural networks and recent heterophily-aware baselines, especially in strongly heterophilic regimes.
\end{abstract}

\section{Introduction}
\label{Section1}
\textbf{Background.} Learning on heterogeneous graphs has become an important paradigm for modeling complex systems with multiple node and relation types, such as academic networks, recommender systems, and knowledge graphs \cite{shi2022heterogeneous}. By encoding type-specific semantics and relation patterns, heterogeneous graph neural networks (HGNNs)~\cite{HAN,zhang2019heterogeneous,survey} have achieved promising results on tasks including node classification, link prediction, and recommendation.

\textbf{Motivation.}  A common design in existing HGNNs is to aggregate messages along metapaths or meta-relations, weighting neighbors mainly by feature similarity. 
This design implicitly assumes that similar features indicate similar labels or semantic roles. 
However, many real-world heterogeneous graphs exhibit pronounced heterophily, where connected nodes often have dissimilar labels or play different roles~\cite{zheng2022graph,zhu2023heterophily,latgrl,gong2024towards}. Under such conditions, feature similarity along metapaths becomes an unreliable proxy for relational semantics, and naive feature-similarity-driven aggregation can propagate misleading information and blur decision boundaries. This challenge is further exacerbated by the coexistence of multiple node types and complex inter-type relations, where homophilic and heterophilic interactions are intertwined within local neighborhoods. 
{
\begin{wrapfigure}{r}{0.46\textwidth}
    \centering     \includegraphics[width=\linewidth]{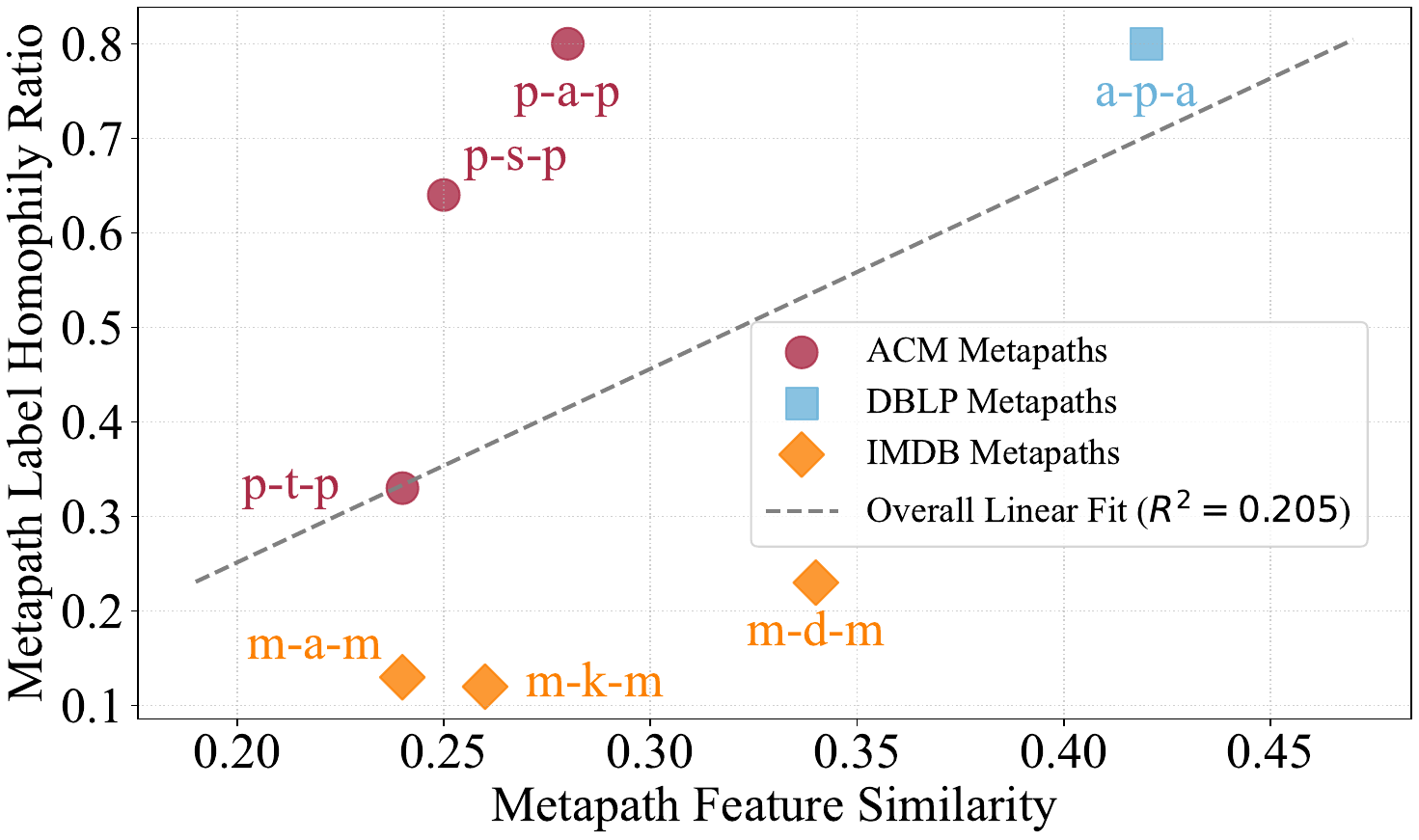}
    \vspace{-6.0mm}
    \caption{Metapath feature similarity vs.\ label homophily ratio across datasets. Each point denotes one metapath, and the dashed line shows the global linear fit ($R^2 = 0.205$).}
    \label{fig1}
\end{wrapfigure}
Figure~\ref{fig1} shows a \textit{\textbf{weak}} correlation between metapath-based feature similarity and label homophily, suggesting that feature similarity explains little variance in label homophily. This indicates that feature-based aggregation alone is insufficient to capture relational semantics in heterogeneous graphs under heterophily. Recent heterophily-aware methods for heterogeneous graphs \cite{Hetero2Net,latgrl}, such as metapath-level heterophily metrics and auxiliary prediction tasks, mark important progress. However, they still rely on feature-based message passing over metapath views and do not explicitly decouple heterogeneous semantics from heterophilic structures or learn node-wise structure-aware fusion of homophilic and heterophilic signals.

}
\noindent\textbf{Method.} In this paper, we propose HeterSEED, a semantics–structure decoupling framework with explicit heterophily-awareness for heterogeneous graphs. Specifically, HeterSEED decomposes representation learning into two complementary channels: (i) a \emph{Heterogeneous Semantic Aggregation Channel} that captures type-specific semantic information across multiple node and relation types via metapath- or relation-based aggregation, and (ii) a \emph{Structure-Aware Heterophily Channel} that leverages pseudo labels and metapath-based structural weights to separate homophilic and heterophilic neighborhoods and aggregate them through structure-guided operators. A node-level structure-aware fusion mechanism then combines the two channels by by learning fusion weights from the concatenated semantic and structural embeddings, enabling adaptive balancing of the two channels and robust node representations under heterophily. From a theoretical perspective, we show that this semantics–structure decoupling strictly improves the expressive power over standard HGNNs on heterophilic heterogeneous graphs (see \underline{Theorem \ref{thm:expressiveness_main}}), reduces the prediction bias introduced by heterophilic neighbors (see \underline{Proposition \ref{prop:bias_main}}).

\textbf{Summary of Contributions.} The main contributions of this work are three-fold:
\begin{itemize}
    \item We propose \textbf{HeterSEED}, a heterophily-aware semantics–structure decoupling framework for learning on heterogeneous graphs. Unlike existing HGNNs that primarily rely on feature-similarity-driven aggregation along metapaths, HeterSEED explicitly separates heterogeneous semantic modeling from heterophilic structural modeling.

    \item We provide rigorous \textbf{theoretical analysis} of HeterSEED from two complementary perspectives. The hypothesis class of HeterSEED is shown to be strictly more expressive than that of standard feature-similarity-driven HGNNs on heterophilic heterogeneous graphs, the separation of homophilic and heterophilic neighborhoods provably reduces the prediction bias introduced by heterophilic edges.

    \item We conduct extensive experiments on five real-world heterogeneous graphs, including two \textbf{large-scale} networks at the \textit{million} node and \textit{hundred-million} edge scale, and show that HeterSEED generally outperforms representative HGNNs and recent heterophily-aware baselines on node classification.

\end{itemize}

\section{Preliminaries and Notation}
\label{sec:pre}
\begin{definition}[\underline{Heterogeneous Graph}~\cite{Sun2012MiningHI}]
A \textit{heterogeneous graph} is a directed graph 
$\mathcal{G} = (V, E, T, R, \phi, \psi)$, 
where $V$ and $E$ are the sets of nodes and edges, and 
$T$ and $R$ are the sets of node types and edge types. 
Each node $v_i \in V$ is associated with a type 
$\phi(v_i) = t_i \in T$, and each edge $e_{ij} \in E$ is associated with a type 
$\psi(e_{ij}) = r_{ij} \in R$. 
The graph is \textit{heterogeneous} if $|T| + |R| > 2$. 
Each node $v_i$ is equipped with a feature vector $x_i \in \mathbb{R}^{d_i}$. 
Let $A^{t_i,t_j}$ denote the adjacency matrix between node types $t_i$ and $t_j$, where 
$A^{t_i,t_j}_{uv} = 1$ indicates that $\phi(u) = t_i$, $\phi(v) = t_j$, and $u$ and $v$ are connected.
\end{definition}

\begin{definition}[\underline{Metapath}~\cite{Sun2012MiningHI}]\label{defi_metapath}
A metapath $p$ is a sequence of node and edge types of the form
$p = t_1 \xrightarrow{r_1} t_2 \xrightarrow{r_2} \cdots \xrightarrow{r_l} t_{l+1}$, where $t_1, \ldots, t_{l+1} \in T$ and $r_1, \ldots, r_l \in R$. A metapath encodes a composite semantic relation pattern between node types and corresponds to multiple meta-path instances in $\mathcal{G}$. In this work, we focus on \textit{symmetric} meta-paths where $t_1 = t_{l+1}$.
\end{definition}
\begin{definition}[\underline{Metapath-based Graph and Neighbor}s~\cite{fu2020magnn}]
Given a meta-path $p$ in a heterogeneous graph $\mathcal{G}$, 
the \textit{meta-path-based neighbors} of a node $v$ are the nodes that can be reached from $v$ through instances of $p$ (including $v$ itself when $p$ is symmetric). 
The \textit{meta-path-based graph} $\mathcal{G}_p$ is constructed by linking all pairs of meta-path-based neighbors. 
The graph $\mathcal{G}_p$ is \textit{homogeneous} if the head and tail node types of $p$ are identical. 
We denote the neighbors of $v$ in $\mathcal{G}_p$ as $\mathcal{N}^{v}_{\mathcal{G}_p}$.
\end{definition}
\begin{definition}[\underline{Homophily Ratio in Heterogeneous Graphs}\cite{latgrl}]
Given a symmetric metapath $p = t_1 \xrightarrow{r_1} t_2 \xrightarrow{r_2} \cdots \xrightarrow{r_l} t_{l+1}$, with $t_1 = t_{l+1}$, the metapath–induced graph $\mathcal{G}_p$ is a homogeneous graph constructed based on $p$. 
Let $E_p$ denote the edge set of $\mathcal{G}_p$, and let $y_i$ be the label of node $v_i$ when supervision is available. 
The \textit{global homophily ratio} of $\mathcal{G}_p$ is defined as
\begin{equation}
\mathcal{H}(\mathcal{G}_p) = 
\frac{
\bigl| \{ (v_i, v_j) \in E_p : y_i = y_j \} \bigr|
}{
|E_p|
}.
\end{equation}
To obtain the overall homophily level of the heterogeneous graph $\mathcal{G}$, 
we take the average homophily ratio across all symmetric meta-path–induced subgraphs:
\vspace{-1mm}
\begin{equation}
\mathcal{H} = 
\frac{1}{|\mathcal{P}|}
\sum_{p \in \mathcal{P}} \mathcal{H}(\mathcal{G}_p),
\label{eq:hetero_homophily}
\end{equation}
where $\mathcal{P}$ denotes the set of all symmetric metapaths considered. 
This formulation reflects the average label consistency of metapath–based homogeneous projections and provides a unified measure of homophily for heterogeneous graphs.
\end{definition}

\section{Method: HeterSEED}
\label{sec:method}
\subsection{Framework Overview}\label{sec:framework}
\begin{figure*}[htbp!]
    \centering
    \includegraphics[width=0.99\textwidth]{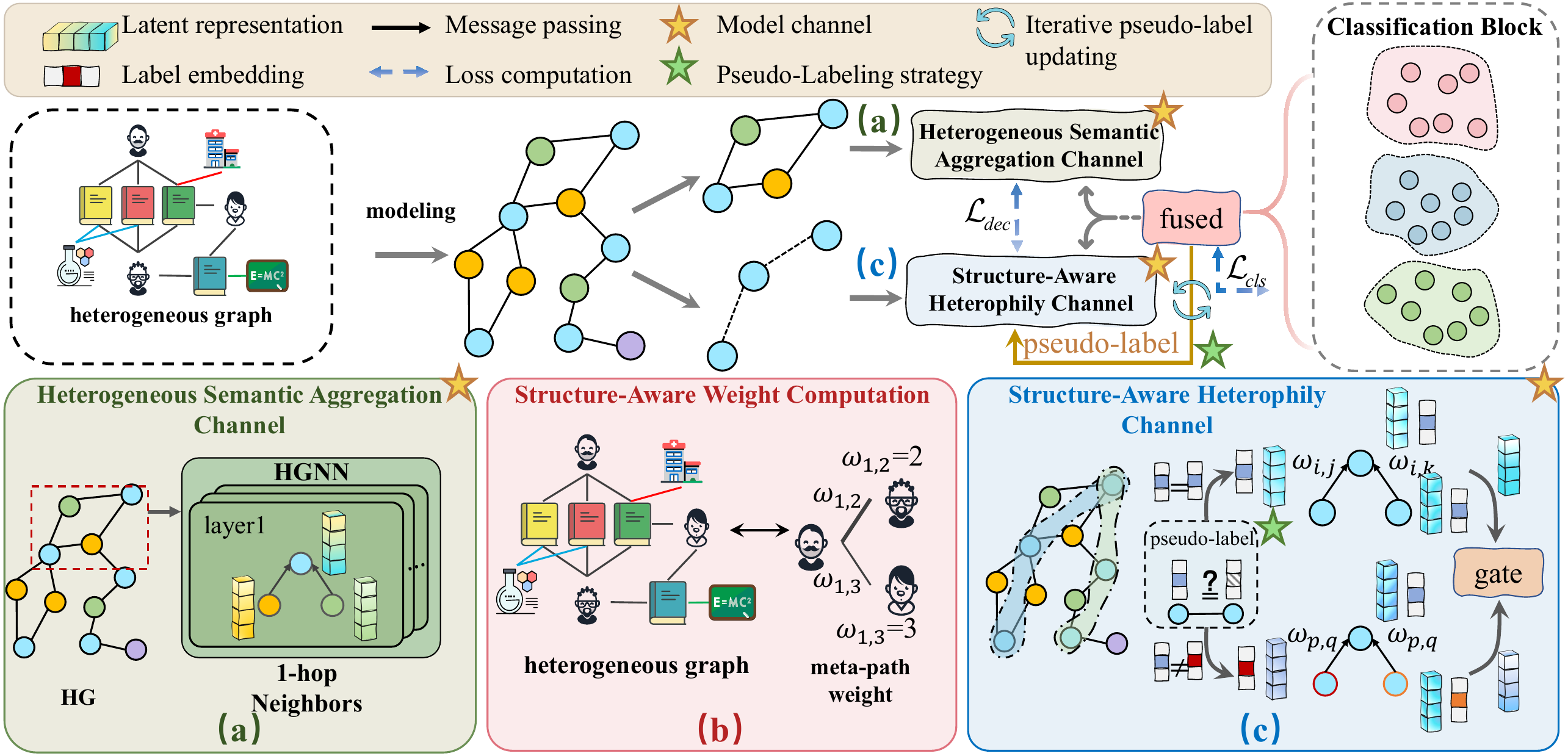}\vspace{-1mm}
    \caption{Schematic of the proposed \textbf{HeterSEED} framework.}\vspace{-4mm} 
    \label{fig:framework}
\end{figure*}
We introduce HeterSEED, a heterophily-aware semantics–structure decoupling framework for learning node representations on heterogeneous graphs. As illustrated in Figure~\ref{fig:framework}, starting from the original heterogeneous graph, a generic HGNN backbone defines a \emph{Heterogeneous Semantic Aggregation Channel} (Figure~\ref{fig:framework}\textbf{(a)}) that aggregates messages from first-order neighbors across different node and relation types to produce type-specific semantic embeddings. 
In parallel, a \emph{Structure-Aware Heterophily Channel} (Figure~\ref{fig:framework}\textbf{(c)}) operates on metapath-based views, where a structure-aware weight computation module (Figure~\ref{fig:framework}\textbf{(b)}) and iteratively refined pseudo-labels are used to separate homophilic and heterophilic neighborhoods and encode their structural patterns into heterophily-aware embeddings. 
A node-level gating module then fuses the semantic and structural embeddings, and the fused representations are fed into a classification block trained with both a classification loss and a decoupling loss to obtain semantics–structure decoupled, heterophily-aware node representations.

\subsection{Heterogeneous Semantic Aggregation Channel}
The heterogeneous semantic aggregation channel in HeterSEED is designed to learn type-aware semantic representations $\mathbf{h}_v^{(s)}$ by aggregating one-hop neighbor information across different node and relation types. This channel follows a standard HGNN-style message-passing paradigm and focuses on capturing local semantic consistency induced by the heterogeneous schema, without relying on handcrafted metapaths or explicit heterophily modeling. 

Formally, let $\mathcal{N}_r(v)$ denote the set of one-hop neighbors of node $v$ connected via relation type $r \in R$, and let $\mathbf{h}_v^{(l)}$ be the representation of $v$ at layer $l$ (with $\mathbf{h}_v^{(0)} = \mathbf{x}_v$). At the $(l+1)$-th layer, we perform neighborhood-wise convolution over heterogeneous neighbors as
\vspace{-1mm}
\begin{equation*}
\mathbf{h}_v^{(l+1)} = \sigma \!\bigg(
\mathbf{W}_{\text{self}}^{(l)} \mathbf{h}_v^{(l)}
+ \sum_{r \in R} \frac{1}{|\mathcal{N}_r(v)|}
\sum_{u \in \mathcal{N}_r(v)} \mathbf{W}_{r}^{(l)} \mathbf{h}_u^{(l)}
\bigg),\vspace{-0.5mm}
\end{equation*}
where $\mathbf{W}_{\text{self}}^{(l)}$ is a learnable weight matrix for self-transformation, $\mathbf{W}_{r}^{(l)}$ is a relation-specific transformation for messages coming from neighbors connected by relation $r$, and $\sigma(\cdot)$ is a non-linear activation function (e.g., ReLU). The inner summation aggregates messages from neighbors of type-$r$ in a permutation-invariant manner (we use mean aggregation by default), and the outer summation combines contributions from all relation types. Stacking $L$ such layers yields the heterogeneous semantic embedding $\mathbf{h}_v^{(s)} = \mathbf{h}_v^{(L)}$ for each node $v$, which encodes type- and relation-aware semantic information and will later be fused with the heterophilic structural embedding produced by the structure-aware heterophily channel.
\subsection{Structure-Aware Heterophily Channel}
The structure-aware heterophily channel explicitly models relational patterns that are driven by structure rather than feature similarity, producing a heterophily-aware structural embedding $\mathbf{h}_v^{(r)}$ for each node. It leverages pseudo-labels, metapath-induced structural weights, and a gated two-branch aggregation over homophilic and heterophilic neighbors, as detailed in the following subsections.

\noindent\textbf{Homophilic and Heterophilic Neighbor Construction.} We first assign a pseudo-label to each node to characterize label consistency among connected nodes. 
For labeled nodes in the training set, we directly use ground-truth labels, while for unlabeled nodes we infer pseudo-labels from the predictions of a linear classifier applied to the current embeddings (updated iteratively during training). 
Formally, let $\mathbf{y}$ denote the ground-truth labels and $\hat{\mathbf{y}}$ the predicted labels. 
The pseudo-label of node $v$ is defined as
\begin{equation}
y^{\text{pseudo}}_v =
\begin{cases}
y_v, & \text{if } v \in \mathcal{V}_{\text{train}},\\
\hat{y}_v, & \text{otherwise},
\end{cases}
\label{eq:pseudo_def}
\end{equation}
where $\mathcal{V}_{\text{train}}$ denotes the set of nodes with available labels.

Given a symmetric metapath $p \in \mathcal{P}$ and its induced homogeneous graph $\mathcal{G}_p = (V_p,E_p)$, we decompose $E_p$ into homophilic and heterophilic edge sets according to pseudo-label consistency:
\vspace{-1.5mm}
\begin{equation}
E_{\text{homo}}^p = \{ (u,v) \in E_p \mid y^{\text{pseudo}}_u = y^{\text{pseudo}}_v \},\qquad
E_{\text{hetero}}^p = \{ (u,v) \in E_p \mid y^{\text{pseudo}}_u \neq y^{\text{pseudo}}_v \}.\vspace{-1.5mm}
\end{equation}
The corresponding node sets contain only nodes incident to these edges,
\vspace{-1.5mm}
\begin{equation}
V_{\text{homo}}^p = \{ v \mid \exists u,\ (u,v)\in E_{\text{homo}}^p \},\qquad
V_{\text{hetero}}^p = \{ v \mid \exists u,\ (u,v)\in E_{\text{hetero}}^p \},
\end{equation}
yielding two structure-induced subgraphs 
$\mathcal{G}_{\text{homo}}^p = (V_{\text{homo}}^p, E_{\text{homo}}^p)$ and
$\mathcal{G}_{\text{hetero}}^p = (V_{\text{hetero}}^p, E_{\text{hetero}}^p)$ that capture homophilic and heterophilic relational patterns under metapath $p$.
Aggregating over all symmetric metapaths leads to
$\mathcal{G}_{\text{homo}} = \bigcup_{p \in \mathcal{P}} \mathcal{G}_{\text{homo}}^p$ and
$\mathcal{G}_{\text{hetero}} = \bigcup_{p \in \mathcal{P}} \mathcal{G}_{\text{hetero}}^p$.
We then define the homophilic and heterophilic neighbor sets of node $v$ with respect to these aggregated graphs as
\vspace{-1.5mm}
\begin{equation}
\mathcal{N}^{v}_{\text{homo}} = \{ u \in \mathcal{N}_{\mathcal{G}_{\text{homo}}}(v) \mid
y^{\text{pseudo}}_u = y^{\text{pseudo}}_v \},\qquad
\mathcal{N}^{v}_{\text{hetero}} = \{ u \in \mathcal{N}_{\mathcal{G}_{\text{hetero}}}(v) \mid
y^{\text{pseudo}}_u \neq y^{\text{pseudo}}_v \},
\end{equation}
where $\mathcal{N}_{\mathcal{G}_{\text{homo}}}(v)$ and $\mathcal{N}_{\mathcal{G}_{\text{hetero}}}(v)$
denote the neighbors of $v$ in the aggregated homophilic and heterophilic graphs, respectively. 
These neighbor sets provide a heterophily-aware structural view that is decoupled from the semantic aggregation in the first channel.


\begin{wrapfigure}{r}{0.45\textwidth}
    \centering
    \includegraphics[width=\linewidth]{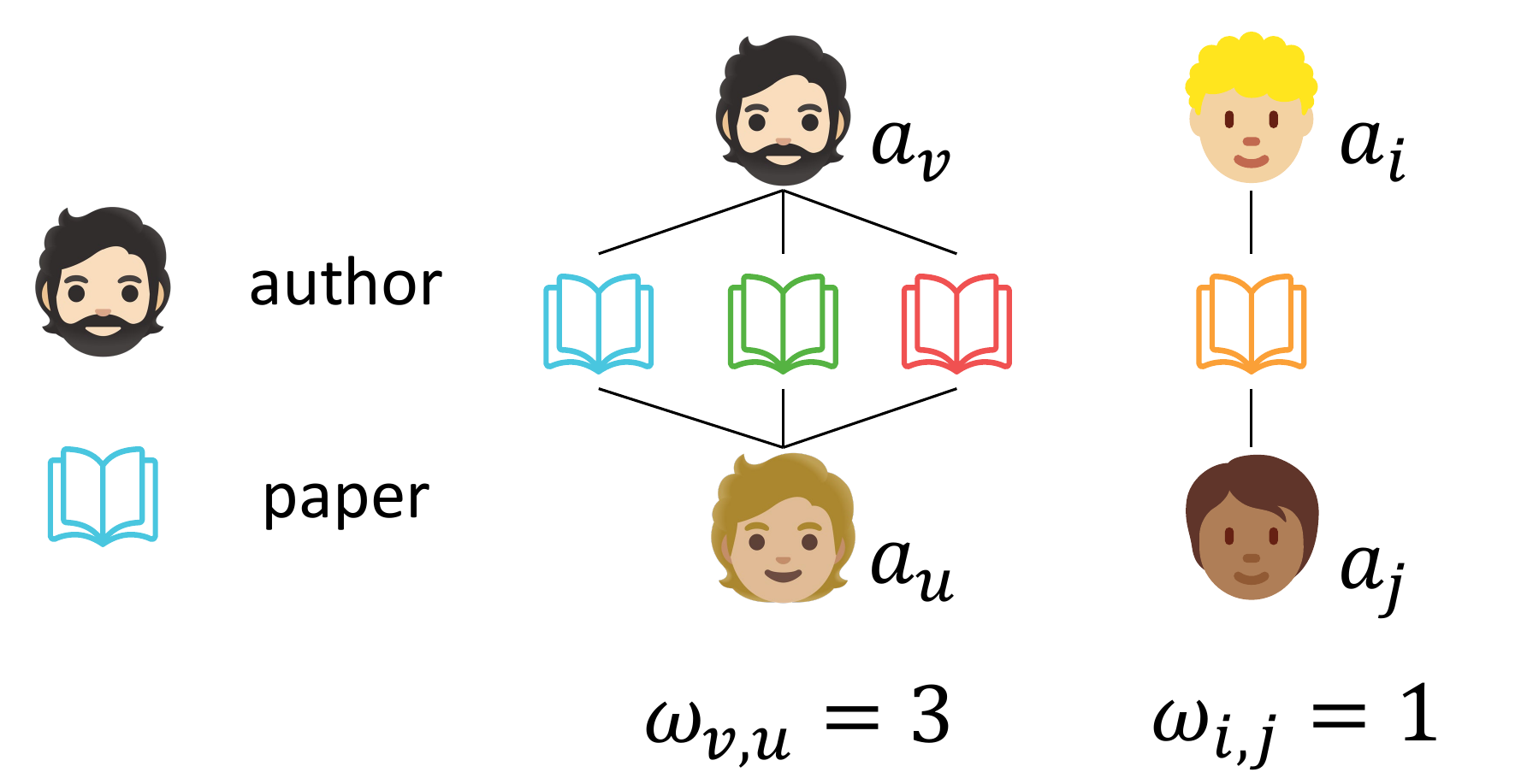}
    \vspace{-3mm}
    \caption{Illustration of metapath-based structural weights.}
    \label{fig:weight_illu}
\end{wrapfigure}

\noindent\textbf{Structure-Aware Weight Computation.} Following Definition~\ref{defi_metapath}, each symmetric metapath $p \in \mathcal{P}$ encodes a type of higher-order semantic and structural correlation.
For any node pair $(u,v)$, let $C_p(u,v)$ denote the number of metapath instances connecting them under $p$.
We define the raw structural weight as
$\omega_{u,v} = \sum_{p \in \mathcal{P}} C_p(u,v)$, which captures the overall strength of higher-order connectivity between $u$ and $v$ across all considered metapaths.
Intuitively, as illustrated in Figure~\ref{fig:weight_illu}, if nodes represent authors and edges encode co-authorship, then node pairs with more co-authored papers have larger $C_p(u,v)$ under the \textsc{APA} metapath, leading to higher $\omega_{u,v}$ and indicating stronger structural relation.
To quantify the relative importance of neighbors for node $v$, we normalize the raw weights via a softmax over its metapath-induced neighbors:
\vspace{-2mm}
\begin{equation}
\label{equ:normalize}
\tilde{\omega}_{u,v} =
\frac{\exp(\omega_{u,v})}
{\sum_{u' \in \mathcal{N}_{\mathcal{G}}(v)} \exp(\omega_{u',v})},
\end{equation}
where $\mathcal{N}_{\mathcal{G}}(v)$ denotes the set of metapath-induced neighbors of $v$. $\tilde{\omega}_{u,v}$ serves as a structure-aware edge weight, measuring the importance of neighbor $u$ relative to $v$. This softmax normalization not only ensures numerical stability, but also performs target-wise calibration by converting raw metapath counts into comparable weights, thereby emphasizing relatively stronger structural connections for each node.

\noindent\textbf{Homophily--Heterophily Structural Fusion Layer.} To explicitly differentiate and integrate homophilic and heterophilic structural information,
we design a homophily--heterophily structural fusion layer.
This layer operates directly on the input node representations $\mathbf{x}_u$
(e.g., raw features or a shared base embedding),
and re-aggregates them over the homophilic and heterophilic neighbor sets
using the structure-aware weights.
First, two branch-specific structural aggregations are performed:
\vspace{-1mm}
\begin{equation*}
\label{eq:homo_hetero_agg}
\mathbf{h}_{\text{homo}}^{(v)} =
\sum_{u \in \mathcal{N}^{v}_{\text{homo}}}
\tilde{\omega}_{u,v} \, \mathbf{x}_u,\quad
\mathbf{h}_{\text{hetero}}^{(v)} =
\sum_{u \in \mathcal{N}^{v}_{\text{hetero}}}
\tilde{\omega}_{u,v} \, \mathbf{x}_u.
\end{equation*}
Next, the two structural signals are fused in a dimension-wise adaptive manner. 
Let $\mathcal{T}_{\text{homo}}(\cdot)$ and $\mathcal{T}_{\text{hetero}}(\cdot)$ denote two learnable transformation operators, and let $\mathbf{g}_v \in \mathbb{R}^{d}$ be a learnable gating vector for node $v$. 
The fused structural embedding is computed as
\vspace{-2mm}
\begin{equation*}
\label{eq:struct}
\mathbf{h}_v^{(r)}
=
\sigma(\mathbf{g}_v) \odot
\mathcal{T}_{\text{homo}}\!\bigl(\mathbf{h}_{\text{homo}}^{(v)}\bigr)
+
\bigl(1 - \sigma(\mathbf{g}_v)\bigr) \odot
\mathcal{T}_{\text{hetero}}\!\bigl(\mathbf{h}_{\text{hetero}}^{(v)}\bigr),\vspace{-2mm}
\end{equation*}
where $\sigma(\cdot)$ is the sigmoid function and $\odot$ denotes element-wise multiplication. 

\noindent\textbf{Implementation of Transformation Operators.} In our implementation, the transformation operators $\mathcal{T}_{\text{homo}}(\cdot)$ and $\mathcal{T}_{\text{hetero}}(\cdot)$ are implemented by two-layer MLPs with ReLU activations and a shared output dimension. More expressive designs, such as attention-based blocks or other message-passing architectures, are in principle possible, but we use MLPs to keep the structural fusion layer simple and efficient while isolating the effect of homophily–heterophily decoupling. These operators project the homophilic and heterophilic representations into a shared latent space before gated fusion. The resulting $\mathbf{h}_v^{(r)}$ serves as the heterophily-aware structural embedding of node $v$, which is later fused with the semantic embedding $\mathbf{h}_v^{(s)}$ from the first channel to form the final representation for downstream prediction.

\subsection{Label-Guided Masking Mechanism}
Let $\mathbf{Z}_{\text{label}} \in \mathbb{R}^{(C+1) \times d}$ be a learnable label embedding matrix, where $C$ is the number of classes, $d$ is the embedding dimension, and the last row $\mathbf{Z}_{\text{label}}[C+1]$ corresponds to a special \texttt{[MASK]} token. 
For each labeled node $i$ with label index $\ell_i \in \{1,\dots,C\}$, we sample a masked label embedding during training:
\vspace*{-1mm}
\begin{equation}
\tilde{\mathbf{z}}_i^{(\text{label})} =
\begin{cases}
\mathbf{Z}_{\text{label}}[\ell_i], & \text{with probability } 1-\beta,\\[2mm]
\mathbf{Z}_{\text{label}}[C+1], & \text{with probability } \beta,
\end{cases}
\end{equation}
where $\beta \in [0,1]$ is a hyperparameter that controls the masking rate of observed labels. 
For unlabeled nodes, no label embedding is used and their features remain unchanged.

The masked label embedding is incorporated into the node feature via a learnable projection 
$g: \mathbb{R}^{d} \rightarrow \mathbb{R}^{d_{\text{feat}}}$: $\tilde{\mathbf{x}}_i = \mathbf{x}_i + g\big(\tilde{\mathbf{z}}_i^{(\text{label})}\big)$, where $d_{\text{feat}}$ is the node feature dimension. 
The resulting $\tilde{\mathbf{x}}_i$ are then fed into the semantic and structural channels of HeterSEED. 
During inference, masking is disabled and the original label embedding is deterministically used for labeled nodes: $\tilde{\mathbf{x}}_i = \mathbf{x}_i + g\big(\mathbf{Z}_{\text{label}}[\ell_i]\big)$.
This design treats label embeddings as feature-level anchors for labeled training nodes, rather than as signals to be propagated across the graph. The stochastic masking mechanism acts as a form of regularization, preventing the model from over-relying on visible labels during training. Disabling masking at inference removes this stochasticity without introducing any additional supervision, ensuring that no extra label information beyond the observed training labels is used at test time.

\subsection{Node-Level Semantics--Structure Fusion}


\noindent\textbf{Gating-Based Fusion.}
To determine the node-dependent importance of the structural signal, we first concatenate the two channel outputs to form a joint feature
$\mathbf{z}_v = \big[\, \mathbf{h}_v^{(s)} \,\Vert\, \mathbf{h}_v^{(r)} \,\big]$.
This vector is fed into a single-layer feed-forward network to produce a scalar gate $\gamma_v = \sigma\!\left(f_{\text{gate}}(\mathbf{z}_v)\right)$, where $f_{\text{gate}}(\cdot)$ is a learnable linear transformation and $\sigma(\cdot)$ is the sigmoid function.
The final fused embedding of node $v$ is then obtained as
\begin{equation}
\label{equ:adaptive}
\mathbf{h}_v =
(1 - \gamma_v)\,\mathbf{h}_v^{(s)}
+
\gamma_v\,\mathbf{h}_v^{(r)}.
\end{equation}
This node-level fusion mechanism endows the model with local adaptivity, allowing it to dynamically adjust the contribution of the semantic and structural channels based on each node’s neighborhood context and feature distribution.

\noindent\textbf{Decoupling Regularization.}
To further encourage semantics--structure decoupling, we introduce a regularization term that penalizes excessive statistical dependency between the two representation spaces. 
Let $\mathbf{H}^{(s)} \in \mathbb{R}^{B \times d}$ and
$\mathbf{H}^{(r)} \in \mathbb{R}^{B \times d}$ denote the semantic and structural embeddings of a mini-batch of $B$ nodes.
We define the decoupling loss as
\vspace{-1mm}
\begin{equation}
\label{eq:dec}
\mathcal{L}_{\text{dec}} =
\operatorname{DecoupleLoss}\!\left(\mathbf{H}^{(s)}, \mathbf{H}^{(r)}\right),
\end{equation}

\vspace{-3mm}

which encourages the two spaces to capture complementary yet weakly correlated information.
In practice, $\operatorname{DecoupleLoss}(\cdot,\cdot)$ can be instantiated by minimizing, for example, the average cosine similarity or the cross-covariance between the two embedding matrices.

\subsection{Training Objective}

For the downstream node classification task, we adopt the standard cross-entropy loss:
\vspace{-1mm}
\begin{equation}
\label{eq:cls}
\mathcal{L}_{\text{cls}} =
\frac{1}{|\mathcal{V}_{\text{train}}|}
\sum_{u \in \mathcal{V}_{\text{train}}}
\ell_{\text{CE}}(\tilde{y}_u, y_u),\vspace{-1mm}
\end{equation}

\vspace{-3mm}

where $\mathcal{V}_{\text{train}}$ denotes the set of labeled training nodes and
$\ell_{\text{CE}}(\cdot,\cdot)$ is the cross-entropy loss (for multi-label classification we use binary cross-entropy instead).
The overall training objective of HeterSEED is
\begin{equation}
\label{eq:loss}
\mathcal{L}
=
\mathcal{L}_{\text{cls}}
+
\alpha \mathcal{L}_{\text{dec}},
\end{equation}
where $\alpha$ is a hyperparameter that controls the strength of the semantics--structure decoupling regularization.

\section{Theoretical Insights}
\label{sec:theory}

\noindent\textbf{Setup and Hypothesis Class.} Let $\mathcal{G} = (V,E)$ be a heterogeneous graph with node features $\mathbf{X} = (\mathbf{x}_v)_{v \in V}$ and labels $y_v \in \{1,\dots,C\}$.
The heterogeneous semantic aggregation channel of HeterSEED produces embeddings
$\mathbf{h}_v^{(s)} = f_s(\mathcal{G}, \mathbf{X}; \theta_s)$,
while the structure-aware heterophily channel produces
$\mathbf{h}_v^{(r)} = f_r(\mathcal{G}_{\text{homo}}, \mathcal{G}_{\text{hetero}}, \mathbf{X}; \theta_r)$,
where $\mathcal{G}_{\text{homo}}$ and $\mathcal{G}_{\text{hetero}}$ are the homophilic and heterophilic graphs constructed in Section~3.
The node-level fusion mechanism yields the final embedding
\begin{equation}
\mathbf{h}_v =
(1-\gamma_v)\,\mathbf{h}_v^{(s)} + \gamma_v\,\mathbf{h}_v^{(r)},
\end{equation}
where $\gamma_v = \sigma(f_{\text{gate}}([\mathbf{h}_v^{(s)} \Vert \mathbf{h}_v^{(r)}]))$ is a node-dependent gate and a linear classifier is applied on $\mathbf{h}_v$ to predict labels.
We denote by $\mathcal{F}_s$ the hypothesis class of models that only use the semantic channel (standard HGNNs on heterogeneous graphs), and by $\mathcal{F}_{\text{HeterSEED}}$ the class induced by the full HeterSEED architecture.

\subsection{Expressive Power of $\mathcal{F}_{\text{HeterSEED}}$}

\begin{theorem}[\underline{Expressiveness Gain}]
\label{thm:expressiveness_main}
There exists a heterogeneous graph $\mathcal{G}$ with node features $\mathbf{X}$ and labels $y$ such that:i ) For any model $f_s \in \mathcal{F}_s$ with bounded depth and width, the classification error satisfies
    $R(f_s) \ge \varepsilon_0$ for some $\varepsilon_0 > 0$; ii) There exists a HeterSEED model $f \in \mathcal{F}_{\text{HeterSEED}}$ with finite depth and width such that $R(f) = 0$.
\end{theorem}
\noindent\textbf{Remark.} Theorem~1 shows that on certain strongly heterophilic heterogeneous graphs, standard HGNNs based on 1-hop feature aggregation cannot separate the classes, whereas HeterSEED can achieve zero error. The key is that HeterSEED leverages \(A\!P\!A\) metapath structural counts and homophily/heterophily separation to recover label information beyond purely feature-based aggregation; the formal construction and proof are given in \textbf{Appendix~\ref{app:proof_expressiveness}}.

\subsection{Bias Reduction Under Heterophily}
We next study how separating homophilic and heterophilic neighbors reduces the bias introduced by heterophilic edges. 
For clarity, we work in a squared-loss regression setting with scalar labels $y_v \in \{-1,+1\}$ and linear aggregation (the classification case is analogous; see \textbf{Appendix~\ref{app:proof_bias}}). 
Let $q_v$ denote the fraction of heterophilic neighbors of node $v$ in the metapath-induced graph, and let $\tilde{q}_v$ denote the effective fraction after applying the structure-aware heterophily channel and node-level fusion in HeterSEED, which down-weights heterophilic neighbors via structure-aware weights and gating.

\begin{proposition}[\underline{Bias Reduction via Homophily/Heterophily Separation}]
\label{prop:bias_main}
Assume that (i) the pseudo-labels used to construct $\mathcal{G}_{\text{homo}}$ and $\mathcal{G}_{\text{hetero}}$ are $\delta$-accurate, i.e., $\mathbb{P}(y_v^{\text{pseudo}} \neq y_v) \le \delta$, and (ii) the gating function $f_{\text{gate}}$ is $L$-Lipschitz.
Let $R(\cdot)$ denote the squared-loss risk.
Then there exist universal constants $c,C>0$ such that, for any purely semantic HGNN $f_s \in \mathcal{F}_s$ and any HeterSEED model $f \in \mathcal{F}_{\text{HeterSEED}}$,
\[
\mathbb{E}[R(f_s)] - \mathbb{E}[R(f)]
\;\ge\;
c\,\mathbb{E}\big[(q_v - \tilde{q}_v)^2\big] - C\delta.
\]
\end{proposition}

\noindent\textbf{Remark.} Proposition~1 formalizes that heterophilic neighbors act as structured noise for feature-similarity-driven HGNNs, so a larger heterophily ratio \(q_v\) implies larger potential bias. By separating homophilic and heterophilic neighbors and reweighting them via structure-aware weights and pseudo-labels, HeterSEED reduces the effective heterophilic mass from \(q_v\) to \(\tilde{q}_v\), thereby lowering bias when pseudo-labels are reasonably accurate; see \textbf{Appendix~\ref{app:proof_bias}}.

\section{Experiments}
We conduct extensive experiments to evaluate the effectiveness of HeterSEED and address the following research questions: i) \textbf{RQ1}: How does HeterSEED perform on node classification compared with state-of-the-art HGNNs and heterophily-aware graph models? ii) \textbf{RQ2}: How do the key components of HeterSEED contribute to its overall performance? 
 iii) \textbf{RQ3}: How does HeterSEED perform across different levels of homophily and heterophily?


\subsection{Datasets}
We evaluate HeterSEED on five real-world heterogeneous graphs, including three widely used benchmarks: DBLP~\cite{SHGN}, IMDB~\cite{SHGN}, and ACM~\cite{HAN}, and two \textbf{large-scale} networks: MAG~\cite{MAG} and RCDD~\cite{AP}. ACM and DBLP are academic citation networks, IMDB is a heterogeneous movie network, while MAG is a large-scale academic graph with 1,939,743 nodes and 21,111,007 edges, and RCDD is a real-world risk commodity detection network with 13,806,619 nodes and 157,814,864 edges.
These two \textbf{large-scale datasets} allow us to assess the scalability of HeterSEED on heterogeneous graphs with millions of nodes and hundreds of millions of edges.
Further details on these datasets are provided in \textbf{Appendix~\ref{app:datasets}}.

\subsection{Baselines and Experimental Settings}
\label{sec:baseline}
We compare HeterSEED with three groups of baselines: (i) state-of-the-art models for heterophilic graphs in the homogeneous setting, including LINKX~\cite{LINKX}, FAGCN~\cite{FAGCN}, ACM-GCN~\cite{ACM-GCN}, and GRAIN~\cite{zhao2025grain}; (ii) representative HGNNs, including RGCN~\cite{RGCN}, RGAT~\cite{RGAT}, HAN~\cite{HAN}, HGT~\cite{HGT}, SHGN~\cite{SHGN}, HINormer~\cite{HINormer}, and DiffGraph~\cite{li2025diffgraph}; and (iii) the heterophily-aware heterogeneous model HETERO$^2$NET~\cite{Hetero2Net}. These baselines cover standard HGNN architectures as well as recent methods tailored to heterophilic graphs. We evaluate all methods on node classification using Macro-F1 and Micro-F1, and additionally report Average Precision (AP) on RCDD in Appendix~\ref{app:rcdd_ap} following~\cite{AP}. All experiments use the standard train/validation/test splits from prior work; each setting is repeated five times with different random seeds, and we report the mean and standard deviation. Additional baseline descriptions, experimental settings, and hyperparameter configurations are provided in \textbf{Appendix~\ref{add_exp}}.



\begin{table*}[htbp!]
\centering
\caption{Performance comparison among HyterSEED and 15 baselines on node classification. The \textbf{bold} and \underline{underline} records indicate the best and the second best results. ``OOM": out of memory.}\vspace{-1mm}
\label{tab:hetero_gnn_comparison}
\resizebox{0.99\textwidth}{!}{
\begin{tabular}{l cc cc cc cc cc}
\toprule
\multirow{2}{*}{\diagbox[width=8em, height=2.4\line]{Method}{Dataset}}& \multicolumn{2}{c}{DBLP} & \multicolumn{2}{c}{IMDB} & \multicolumn{2}{c}{ACM} & \multicolumn{2}{c}{MAG} &\multicolumn{2}{c}{RCDD} \\
\cmidrule(lr){2-3} \cmidrule(lr){4-5} \cmidrule(lr){6-7} \cmidrule(lr){8-9} \cmidrule(lr){10-11}
& Macro-F1 & Micro-F1 & Macro-F1 & Micro-F1 & Macro-F1 & Micro-F1 & Macro-F1 & Micro-F1 & Macro-F1 & Micro-F1 \\

\midrule
\multicolumn{11}{c}{Vanilla Model} \\
\midrule
MLP     & \Data{80.36}{0.31} &\Data{80.92}{0.25} & \Data{57.01}{0.67} & \Data{61.65}{0.93} & \Data{86.50}{0.32} & \Data{86.53}{0.31} & \Data{6.35}{0.38} & \Data{23.60}{0.17} & \Data{86.08}{0.34} & \Data{97.27}{0.43}\\
\rowcolor{rowgray}
GCN     & \Data{90.84}{0.32} & \Data{91.47}{0.34} & \Data{57.88}{1.18} & \Data{64.82}{0.64} & \Data{92.17}{0.24} & \Data{92.12}{0.23} & \Data{25.14}{0.33} & \Data{47.26}{0.36}  & \Data{91.30}{0.38} & \Data{98.29}{0.12}\\
GAT     & \Data{91.05}{0.76} & \Data{91.73}{0.50} & \Data{58.94}{1.35} & \Data{64.86}{0.43} & \Data{92.26}{0.94} & \Data{92.19}{0.93} & \Data{22.94}{0.49} & \Data{43.79}{0.24}  & \Data{89.75}{0.21} & \Data{98.02}{0.04}\\

\midrule
\multicolumn{11}{c}{Homogeneous Model with Heterophily} \\
\midrule

LINKX   & \Data{75.05}{1.45} & \Data{77.78}{1.59} & \Data{58.98}{0.45} & \Data{62.03}{0.41} & \Data{89.91}{1.14} & \Data{89.69}{1.09} & \Data{14.63}{0.56} & \Data{33.43}{0.22} & OOM & OOM\\
\rowcolor{rowgray}
FAGCN   & \Data{82.40}{0.28} & \Data{83.08}{0.28} & \Data{63.68}{0.56} & \Data{67.49}{0.25} & \Data{89.27}{0.62} & \Data{89.34}{0.78} & \Data{16.87}{0.41} & \Data{37.99}{0.95}  & \Data{90.62}{0.32} & \Data{98.09}{0.07} \\

ACM-GCN & \Data{81.80}{0.32} & \Data{82.91}{0.14} & \UnderlineDataFull{65.44}{0.49} & \Data{68.71}{0.40} & \Data{88.97}{0.85} & \Data{89.21}{0.53} & \Data{13.42}{0.79} & \Data{34.47}{0.70} & \Data{84.89}{0.47} & \Data{94.12}{0.11}\\
\rowcolor{rowgray}
GRAIN & \Data{78.47}{0.07} & \Data{79.48}{0.07} & \Data{59.93}{0.18} & \Data{63.28}{0.16} & \Data{92.78}{0.20} & \Data{92.77}{0.20} & OOM & OOM  & OOM & OOM \\

\midrule
\multicolumn{11}{c}{Heterogeneous Model under Homophily} \\
\midrule
RGCN    & \Data{91.52}{0.50} & \Data{92.07}{0.50} & \Data{61.26}{0.33} & \Data{65.21}{0.73} & \Data{91.95}{0.44} & \Data{91.75}{0.35} & \Data{27.01}{0.21} & \Data{48.80}{0.24}  & \Data{92.25}{0.34} & \Data{98.30}{0.07} \\
\rowcolor{rowgray}
HAN     & \Data{91.67}{0.69} & \Data{92.05}{0.62} & \Data{57.74}{0.96} & \Data{64.63}{0.58} & \Data{90.89}{0.43} & \Data{90.79}{0.43} & \Data{8.94}{0.16} & \Data{26.76}{0.36}  & \Data{87.32}{0.34} & \Data{97.46}{0.07}\\
RGAT    & \Data{92.61}{0.48} & \Data{93.15}{0.49} & \Data{57.85}{0.58} & \Data{62.79}{0.70} & \Data{90.03}{0.56} & \Data{90.40}{0.54} & \Data{24.76}{0.47} & \Data{45.29}{0.40}  & \Data{89.19}{1.20} & \Data{97.93}{0.16}\\
\rowcolor{rowgray}
HGT     & \Data{93.01}{0.24} & \Data{93.49}{0.25} & \Data{63.07}{1.19} & \Data{67.20}{1.61} & \Data{91.12}{0.76} & \Data{91.32}{0.89} & \Data{27.87}{0.30} & \Data{49.19}{0.63}  & \Data{91.04}{0.48} & \Data{98.29}{0.08} \\
SHGN    & \UnderlineDataFull{94.01}{0.24} & \Data{94.20}{0.31} & \Data{63.53}{1.26} & \Data{67.36}{0.57} & \Data{93.42}{0.44} & \Data{93.35}{0.45} & \Data{22.61}{0.40} & \Data{43.68}{0.71} & \Data{88.12}{0.52} & \Data{97.68}{0.10} \\
\rowcolor{rowgray}
HINormer & \Data{93.90}{0.37} & \Data{94.25}{0.38} & \Data{64.11}{1.82} & \Data{67.71}{1.05} & \Data{92.66}{0.73} & \Data{93.34}{0.85} & \Data{25.80}{0.41} & \Data{47.57}{0.61}  & OOM & OOM  \\
DiffGraph & \Data{91.45}{0.69} & \Data{91.72}{0.60} & \Data{58.84}{0.59} & \Data{63.61}{0.57} & \Data{93.87}{0.82} & \Data{93.85}{0.81} & OOM & OOM  & OOM & OOM  \\

\midrule
\multicolumn{11}{c}{Heterophily-aware Heterogeneous Model
} \\
\midrule

HETERO$^2$NET & \Data{93.92}{0.45} & \UnderlineDataFull{94.36}{0.57} & \Data{65.33}{0.72} & \UnderlineDataFull{69.03}{0.78} & \UnderlineDataFull{93.96}{0.88} & \UnderlineDataFull{93.87}{0.76} & \UnderlineDataFull{33.28}{0.32} & \UnderlineDataFull{54.66}{0.25} & \UnderlineDataFull{92.36}{0.19} & \UnderlineDataFull{98.43}{0.05}\\
\rowcolor{spiderblue}
HeterSEED (Ours) & \BestData{94.48}{0.81} & \BestData{94.87}{0.73} & \BestData{66.71}{0.73} & \BestData{70.60}{0.82} & \BestData{94.34}{0.81} & \BestData{94.21}{0.77} & \BestData{34.99}{0.22} & \BestData{56.25}{0.25} & \BestData{93.09}{0.15} & \BestData{98.60}{0.03}\\
\bottomrule
\end{tabular}
}
\end{table*}
\subsection{Overall Performance Comparison \textbf{(RQ1)}}
To address \textbf{RQ1}, we compare HeterSEED with 15 baselines on five benchmark datasets, and report Macro-F1 and Micro-F1 in Table~\ref{tab:hetero_gnn_comparison}. Across all datasets and metrics, HeterSEED consistently achieves the best performance, outperforming vanilla GNNs, heterophily-oriented homogeneous models, heterogeneous GNNs, and the heterophily-aware heterogeneous baseline.  Relative to heterophily-oriented homogeneous models, HeterSEED shows clear improvements, suggesting that modeling heterophily purely on a collapsed homogeneous view cannot fully capture the rich type-specific information present in heterogeneous graphs.  When compared with representative HGNNs, HeterSEED consistently produces higher Macro-F1 and Micro-F1 scores on all five datasets. 
These HGNNs are effective on relatively homophilic settings but their performance degrades on low-homophily datasets such as IMDB, MAG, and RCDD, where feature-similarity-driven aggregation tends to propagate noisy signals. 
In contrast, HeterSEED maintains strong performance, which is consistent with our theoretical result in \textbf{Theorem~\ref{thm:expressiveness_main}} that semantics–structure decoupling yields a strictly more expressive hypothesis class than standard HGNNs on heterophilic heterogeneous graphs.

Among all baselines, HETERO$^2$NET is the most related to our work, as it is explicitly designed for heterogeneous graphs with heterophily.
It frequently achieves the second-best results across datasets, yet HeterSEED consistently surpasses it on all metrics, including challenging cases such as IMDB, MAG, and RCDD. 
Moreover, on the two large-scale datasets MAG and RCDD, where several baselines run out of memory or suffer notable performance drops, HeterSEED remains trainable and achieves the best scores. 
The fact that HeterSEED suffers much less degradation as heterophily increases, while other HGNNs and HETERO$^2$NET become biased toward noisy neighbors, empirically echoes \textbf{Proposition~\ref{prop:bias_main}}, which shows that separating homophilic and heterophilic neighborhoods reduces the prediction bias induced by heterophilic edges. 

Overall, these results indicate that the semantics–structure decoupling with structure-aware fusion not only improves accuracy under heterophily at scale, but also behaves in line with the expressiveness and bias-reduction properties predicted by our theoretical analysis in Section~\ref{sec:theory}.
\subsection{Ablation Study (RQ2)}\label{main_abl}
\begin{wraptable}{r}{0.5\textwidth}
\label{tab:ablation}
    \centering
    \caption{Ablation study on HeterSEED.}
    \label{tab:ablation}
    \resizebox{0.99\linewidth}{!}{
        \begin{tabular}{l|cc|cc}
        \toprule
        \textbf{Dataset} & \multicolumn{2}{c|}{\bfseries DBLP} & \multicolumn{2}{c}{\bfseries IMDB} \\
        \cmidrule(lr){2-3} \cmidrule(lr){4-5} 
        Model $\backslash$ Metric & Macro-F1 & Micro-F1 & Macro-F1 & Micro-F1 \\
        \midrule
        w/o SHC & \Data{93.88}{0.41} & \Data{94.30}{0.69} & \Data{65.79}{0.52} & \Data{69.95}{0.85} \\
        w/o Dec          & \Data{94.16}{0.75} & \Data{94.59}{0.68} & \Data{66.28}{0.49} & \Data{70.29}{0.66}\\
        w/o Homo         & \Data{93.87}{0.54} & \Data{94.23}{0.58} & \Data{66.10}{1.17} & \Data{70.00}{0.82}  \\
        w/o Hetero       & \Data{93.66}{0.74} & \Data{94.13}{0.74} & \Data{66.46}{0.74} & \Data{70.25}{0.68} \\
        w/o Mask         & \Data{93.99}{0.61} & \Data{94.35}{0.60} & \Data{66.11}{0.47} & \Data{69.46}{0.69} \\
        \midrule
        \rowcolor{spiderblue} 
        \textbf{HeterSEED (Ours)} & \BestData{94.48}{0.81} & \BestData{94.87}{0.73} & \BestData{66.71}{0.73} & \BestData{70.60}{0.82}  \\
        \bottomrule
        \end{tabular}
    }
\end{wraptable}
To address \textbf{RQ2}, we conduct an ablation study by comparing HeterSEED with several variants that remove or modify a single component: \textbf{w/o SHC} removes the pseudo-label–guided structural channel and keeps only the one-hop semantic HGNN backbone; \textbf{w/o Dec.} drops the semantics–structure decoupling loss $\mathcal{L}_{\text{dec}}$; \textbf{w/o Homo} and \textbf{w/o Hetero} disable aggregation over homophilic and heterophilic neighbors in the structural channel, respectively; and \textbf{w/o Mask} removes the label-masking mechanism from the input features.

Table~\ref{tab:ablation} reports results on DBLP and IMDB. On both datasets, the full HeterSEED model achieves the best Macro-F1 and Micro-F1, while all ablated variants show clear performance drops. The largest degradations occur for \textbf{w/o SHC} and \textbf{w/o Hetero}, especially on the more heterophilic IMDB dataset, highlighting the importance of explicitly modeling heterophilic neighborhoods through the pseudo-label–guided structural channel. Removing the decoupling loss (\textbf{w/o Dec.}) also consistently hurts performance, indicating that enforcing semantics–structure decoupling is beneficial in practice, in line with our theoretical analysis. Ablation results on the other three datasets exhibit the same trend, i.e., the full HeterSEED model consistently outperforms all variants; detailed results are provided in \textbf{Appendix~\ref{app:ablation}}.

\subsection{Performance under Varying Homophily (RQ3)}\label{rq4_exp}

\begin{wrapfigure}{l}{0.45\textwidth}
    \centering
    \vspace{-4mm} 
    \includegraphics[width=\linewidth]{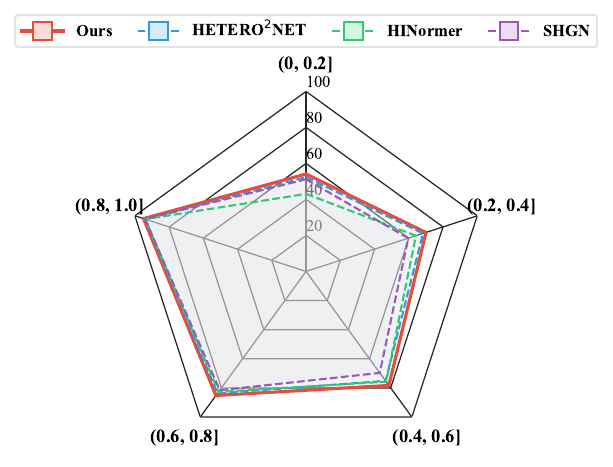}
    \vspace{-4mm} 
    \caption{Performance of HeterSEED and three baselines over node groups with different local homophily ratios on DBLP.} 
    \label{fig:homophily_radar}
    \vspace{-4mm} 
\end{wrapfigure}

To investigate \textbf{RQ3}, we evaluate HeterSEED on node subsets with different homophily levels. For each node, we compute its local homophily ratio (the fraction of neighbors sharing the same label in the metapath-induced graph) and partition nodes into five intervals. We then measure node classification performance within each interval. Figure~\ref{fig:homophily_radar} shows a radar plot comparing HeterSEED with heterophily-oriented baselines HETERO$^2$NET, HINormer, and SHGN. HeterSEED consistently achieves the best performance across all homophily ranges, with the most pronounced gains in the strongly heterophilic regimes $(0,0.2]$ and $(0.2,0.4]$, where other methods suffer from noisy cross-class aggregation. Even in high-homophily intervals, HeterSEED remains at least on par with the strongest baseline, indicating that semantics–structure decoupling does not sacrifice performance when homophily is high. These observations empirically support our bias-reduction analysis in Proposition~\ref{prop:bias_main}, showing that explicitly separating and reweighting homophilic and heterophilic neighbors improves robustness under varying degrees of homophily. Detailed experimental settings are provided in \textbf{Appendix~\ref{app:homophily}}.

\section{Conclusion}
In this paper, we propose HeterSEED, a heterophily-aware semantics–structure decoupling framework for learning on heterogeneous graphs. HeterSEED combines a heterogeneous semantic aggregation channel with a structure-aware heterophily channel that separates homophilic and heterophilic neighborhoods, and fuses them via node-level adaptive gating. We provide theoretical analysis showing that HeterSEED improves expressive power over standard HGNNs and reduces the bias introduced by heterophilic neighbors. Experiments on five heterogeneous graphs, including two large-scale networks, demonstrate that HeterSEED generally achieves leading performance over representative HGNNs and recent heterophily-aware baselines. For future work, it is desirable to extend this semantics–structure decoupling paradigm to heterogeneous hypergraphs under heterophily.


\bibliographystyle{plainnat}
\bibliography{references}

\appendix
\section{Summary of Notation}
Table~\ref{tab:notation} summarizes the main notation used throughout the main text.
\begin{table}[htbp!]
\centering
\caption{Summary of main notation used in this paper.}\label{tab:notation}
\begin{tabular}{ll}
\toprule
Symbol & Description \\
\midrule
$\mathcal{G}=(V,E,T,R,\phi,\psi)$ & Heterogeneous graph with node / edge sets, types and mappings \\
$V, E$ & Sets of nodes and edges \\
$T, R$ & Sets of node types and relation types \\
$\phi(v), \psi(e)$ & Type of node $v$ and edge $e$ \\
$\mathbf{x}_v \in \mathbb{R}^{d_\text{feat}}$ & Input feature of node $v$ \\
$y_v$ & Ground-truth label of node $v$ \\
$\hat{y}_v$ & Predicted label of node $v$ \\
$y_v^{\text{pseudo}}$ & Pseudo-label of node $v$ (Eq.~\eqref{eq:pseudo_def}) \\
$\mathcal{V}_{\text{train}}$ & Set of labeled training nodes \\
\midrule
$p \in \mathcal{P}$ & (Symmetric) metapath, $\mathcal{P}$ the metapath set \\
$\mathcal{G}_p$ & Metapath-induced homogeneous graph for $p$ \\
$\mathcal{N}_{\mathcal{G}_p}(v)$ & Neighbors of $v$ in $\mathcal{G}_p$ \\
$\mathcal{E}_{\text{homo}}^p$, $\mathcal{E}_{\text{hetero}}^p$ & Homophilic / heterophilic edge sets under $p$ \\
$\mathcal{G}_{\text{homo}}$, $\mathcal{G}_{\text{hetero}}$ & Aggregated homophilic / heterophilic subgraphs \\
$\mathcal{N}^{v}_{\text{homo}}$, $\mathcal{N}^{v}_{\text{hetero}}$ & Homophilic / heterophilic neighbors of $v$ \\
\midrule
$\mathbf{h}_v^{(s)}$ & Semantic embedding of node $v$ (semantic channel) \\
$\mathbf{h}_v^{(r)}$ & Structural embedding of node $v$ (heterophily channel) \\
$\mathbf{h}_v$ & Final fused embedding of node $v$ (Eq.~\eqref{equ:adaptive}) \\
$\mathbf{H}^{(s)}, \mathbf{H}^{(r)}$ & Matrices of semantic / structural embeddings for a node batch \\
$\mathbf{W}, \mathbf{W}_r$ & Trainable weights in neighborhood-wise convolution \\
$\omega_{u,v}$, $\tilde{\omega}_{u,v}$ & Raw / normalized structural weight from $u$ to $v$ \\
$\mathbf{h}_{\text{homo}}$, $\mathbf{h}_{\text{hetero}}$ & Aggregated homophilic / heterophilic structural messages \\
$\mathcal{T}_{\text{homo}}(\cdot)$, $\mathcal{T}_{\text{hetero}}(\cdot)$ & two-layer MLPs with a shared output dimension \\
$\gamma_v$ & Node-level fusion gate for combining $\mathbf{h}_v^{(s)}$ and $\mathbf{h}_v^{(r)}$ \\
\midrule
$\mathbf{Z}_{\text{label}} \in \mathbb{R}^{(C+1)\times d}$ & Label embedding matrix with $C$ classes and one \texttt{[MASK]} token \\
$\tilde{\mathbf{z}}_v^{(\text{label})}$ & (Possibly masked) label embedding of node $v$ \\
$\alpha$ & Coefficient of decoupling regularization $\mathcal{L}_{\text{dec}}$ \\
$\beta$ & Masking probability in label-guided masking module \\
$\mathcal{L}_{\text{cls}}$ & Supervised classification loss (cross-entropy) \\
$\mathcal{L}_{\text{dec}}$ & Semantics–structure decoupling regularization loss \\
$\mathcal{L}$ & Overall training objective $\mathcal{L} = \mathcal{L}_{\text{cls}} + \alpha \mathcal{L}_{\text{dec}}$ \\
\bottomrule
\end{tabular}
\end{table}

\section{Related Work}
\label{app:related work}
\subsection{Heterogeneous Graph Neural Networks}
Heterogeneous graphs are ubiquitous in real-world applications, where multiple node and edge types encode rich semantic information~\cite{wang2022survey,shi2022heterogeneous}. To model such structures, numerous heterogeneous graph neural networks (HGNNs) have been proposed, which are broadly categorized into metapath-based and metapath-free approaches~\cite{survey}. Metapath-based methods capture high-order semantic dependencies by propagating and aggregating information along predefined or learned metapaths. Representative examples include MetaPath2Vec~\cite{MetaPath2Vec} for metapath-guided random-walk embeddings, HAN~\cite{HAN} with hierarchical attention over node- and semantic-level contexts, MAGNN~\cite{fu2020magnn} with semantic-specific message passing, and GTN~\cite{GTN}, which learns soft adjacency matrices to discover informative metapaths. SeHGNN~\cite{SeHGNN} further exploits sequential node embeddings to model order-dependent semantics.

Metapath-free HGNNs instead generalize standard message passing to heterogeneous graphs without relying on hand-crafted metapaths. RGCN~\cite{RGCN} and RGAT~\cite{RGAT} introduce relation-specific transformations and attention for different edge types, SHGN~\cite{SHGN} integrates multi-layer attention with learnable edge embeddings, while Transformer-based architectures such as HGT~\cite{HGT} and HINormer~\cite{HINormer} employ self-attention to model heterogeneous interactions. DiffGraph~\cite{li2025diffgraph} adopts latent heterogeneous graph diffusion with cross-view denoising to alleviate noise and capture semantic transitions among heterogeneous relations.

Although these models effectively capture heterogeneous semantics, most are developed under a homophily assumption and mainly rely on feature-similarity-driven aggregation. This design is less reliable on heterogeneous graphs with pronounced heterophily, where neighboring nodes may belong to different semantic categories or play distinct structural roles~\cite{latgrl}. Recent work such as HETERO$^2$NET~\cite{Hetero2Net} begins to explore heterophily-aware representation learning on heterogeneous graphs by introducing multiple representation branches, and other studies consider heterophily-aware message passing across metapaths~\cite{Adaptive} or homophily-guided structure optimization~\cite{Cross-type}. However, these approaches still largely depend on feature-based aggregation or coarse-grained semantic integration and do not explicitly disentangle homophilic and heterophilic structural patterns at a fine granularity, which may limit their ability to exploit discriminative heterophilic structures.

\subsection{Graph Neural Networks with Heterophily}
Graphs with heterophily, where connected nodes often exhibit dissimilar attributes or labels, pose significant challenges for conventional GNNs. A number of methods have been proposed to address heterophily in homogeneous graphs. H2GCN~\cite{H2GCN} separates ego- and neighbor-embeddings, leverages higher-order neighborhoods, and combines intermediate representations to better exploit structural signals. Geom-GCN~\cite{heterophilygeom} introduces a geometric aggregation scheme that preserves neighborhood structure and captures long-range dependencies in disassortative graphs via a latent geometric space. FAGCN~\cite{FAGCN} studies low- and high-frequency components of graph signals and proposes a self-gating mechanism to adaptively integrate them. LINKX~\cite{LINKX} departs from message passing and directly combines node features and adjacency via scalable linear transformations, achieving strong performance on large non-homophilous graphs. ACM-GCN~\cite{ACM-GCN} revisits homophily from a post-aggregation similarity viewpoint and designs an adaptive channel mixing framework that dynamically combines aggregation, diversification, and identity channels. GRAIN~\cite{zhao2025grain} aggregates multi-view information at different granularity levels and introduces an adaptive aggregator to fuse local and global signals, improving robustness across homophily regimes.

These approaches, while effective on homogeneous graphs, do not directly extend to heterogeneous settings, where multiple node types and complex semantic relations substantially increase modeling difficulty. Moreover, they generally lack an explicit mechanism to simultaneously capture heterogeneous semantics and heterophilic structural patterns, often relying on unified aggregation or coarse channel fusion. This gap motivates methods that jointly handle heterogeneity and heterophily in a principled manner.

\subsection{Positioning of HeterSEED}
Our work situates at the intersection of heterogeneous graph learning and heterophilic GNNs. HETERO$^2$NET~\cite{Hetero2Net} is the closest prior work, as it explicitly targets heterophily-aware representation learning on heterogeneous graphs via multiple interaction branches. However, it does not explicitly separate homophilic and heterophilic structural patterns, nor does it perform semantics–structure decoupling with node-wise adaptive fusion. LatGRL~\cite{latgrl} also introduces a heterophily-aware perspective for heterogeneous graphs, but focuses on unsupervised representation learning, which differs from our semi-supervised node classification setting.

In contrast, HeterSEED adopts a dual-channel design that decouples heterogeneous semantic aggregation from structure-aware heterophily modeling. The heterogeneous semantic aggregation channel captures type-specific, relation-aware semantics via first-order heterogeneous neighborhoods, while the structure-aware heterophily channel explicitly distinguishes homophilic and heterophilic neighborhoods, aggregates them with metapath-based structural weights, and encodes discriminative structural patterns. A node-level adaptive fusion module then assigns fusion weights based on the concatenated semantic and structural embeddings, enabling each node to balance semantic consistency and structural discrimination according to its local context. Extensive experiments on multiple heterogeneous benchmarks show that this semantics–structure decoupling leads to consistent gains over state-of-the-art HGNNs and recent heterophily-aware baselines.

\section{Additional Details for Theoretical Analysis}
\label{app:theory}

In this appendix, we provide detailed proofs and supplementary explanations for the theoretical results stated in Section~\ref{sec:theory}.

\subsection{Preliminaries and Function Classes}
\label{app:prelim_theory}

Let $\mathcal{G} = (V,E)$ be a heterogeneous graph with node feature matrix 
$\mathbf{X} = (\mathbf{x}_v)_{v\in V}$, label vector $(y_v)_{v \in V}$, and a schema with node types $T$ and relation types $R$.
For simplicity, we consider node-level prediction and assume that $y_v \in \{1,\dots,C\}$ unless otherwise specified.

\noindent\textbf{Semantic channel.} The heterogeneous semantic aggregation channel of HeterSEED is an HGNN-style message-passing architecture that produces, for each node $v$, an embedding
\[
\mathbf{h}_v^{(s)} 
= f_s(\mathcal{G}, \mathbf{X}; \theta_s)
\in \mathbb{R}^{d_s},
\]
where $\theta_s$ collects all trainable parameters in the semantic channel.
Concretely, $f_s$ is implemented as $L_s$ stacked layers of relation-specific linear transformations, aggregation over type- and relation-specific neighborhoods, followed by pointwise nonlinearities (e.g., ReLU), as described in Section~3.2.

\noindent\textbf{Structure-aware heterophily channel.} The structure-aware heterophily channel constructs metapath-induced homogeneous graphs, separates homophilic and heterophilic edges based on pseudo-labels, and aggregates structure-aware signals with metapath-based structural weights and the homophily--heterophily structural fusion layer.
This channel produces embeddings
\[
\mathbf{h}_v^{(r)} 
= f_r(\mathcal{G}_{\text{homo}}, \mathcal{G}_{\text{hetero}}, \mathbf{X}; \theta_r)
\in \mathbb{R}^{d_r},
\]
where $\mathcal{G}_{\text{homo}}$ and $\mathcal{G}_{\text{hetero}}$ are the homophilic and heterophilic graphs constructed in Section~3.3 and $\theta_r$ collects all trainable parameters of this channel.

\paragraph{Node-level fusion and classifier.}
Given $\mathbf{h}_v^{(s)}$ and $\mathbf{h}_v^{(r)}$, the fusion module computes a node-dependent gate
\[
\gamma_v = \sigma\big(f_{\text{gate}}([\mathbf{h}_v^{(s)} \Vert \mathbf{h}_v^{(r)}])\big)
\in (0,1),
\]
where $f_{\text{gate}}$ is a linear (or shallow MLP) mapping and $\sigma$ is the sigmoid function.
The fused embedding is
\[
\mathbf{h}_v
=
(1-\gamma_v)\,\mathbf{h}_v^{(s)}
+
\gamma_v\,\mathbf{h}_v^{(r)}
\in \mathbb{R}^{d},
\]
and a linear classifier $g$ produces logits
\[
\mathbf{o}_v = W_{\text{cls}} \mathbf{h}_v + \mathbf{b}_{\text{cls}},
\]
followed by softmax to obtain class probabilities.

\noindent\textbf{Function classes.}
We denote by: (i) $\mathcal{F}_s$ the hypothesis class of models that use only the semantic channel and a linear classifier; i.e., models of the form $v \mapsto g_s(f_s(\mathcal{G},\mathbf{X};\theta_s))$, which correspond to standard heterogeneous GNNs; (ii) $\mathcal{F}_{\text{HeterSEED}}$ the hypothesis class of models that use both the semantic channel, the structure-aware heterophily channel, the node-level fusion gate, and a linear classifier; i.e.,$v \mapsto g(f_s(\mathcal{G},\mathbf{X};\theta_s), f_r(\mathcal{G}_{\text{homo}},\mathcal{G}_{\text{hetero}},\mathbf{X};\theta_r))$.
Theorem~\ref{thm:expressiveness_main} in the main text compares the expressive power of $\mathcal{F}_s$ and $\mathcal{F}_{\text{HeterSEED}}$ on a specific class of heterophilic heterogeneous graphs. 

For risk, unless otherwise stated we write
\[
R(f) := \mathbb{E}[\ell(f(v),y_v)],
\]
where $\ell$ is a non-negative loss (e.g., 0-1 loss or squared loss) and the expectation is taken with respect to the data distribution over $(\mathcal{G}, \mathbf{X}, y)$ and a uniformly sampled node $v$.

\subsection{Proof of Theorem~\ref{thm:expressiveness_main}}
\label{app:proof_expressiveness}
\noindent\textbf{Graph construction.} Consider a heterogeneous graph with node types $T = \{A, P\}$ representing authors and papers, respectively, and a relation type $R = \{\text{writes}\}$.
Edges of type $A \to P$ indicate that an author wrote a paper.
Let us focus on the metapath $A \xrightarrow{\text{writes}} P \xrightarrow{\text{writes}^{-1}} A$, abbreviated as $A\!P\!A$. In particular, we construct $2n$ author nodes,
\[
A = \{a_1,\dots,a_n\} \cup \{b_1,\dots,b_n\},
\]
and a set of paper nodes $P$.
Authors $a_i$ are assigned label $+1$ and authors $b_i$ are assigned label $-1$:
\[
y_{a_i} = +1,\quad
y_{b_i} = -1,\quad
i = 1,\dots,n.
\]
All paper nodes are treated as unlabeled (or ignored) in the classification task. We assume that all author nodes share identical features:
\[
\mathbf{x}_{a_i} = \mathbf{x}_{b_j} = \mathbf{x}_0 \in \mathbb{R}^{d_0}, \quad \forall i,j.
\]
Paper nodes can have arbitrary features; they will not break the symmetry we exploit. As for edges and heterophily, we connect each paper node to exactly one positive author and one negative author.
For instance, for each $i=1,\dots,n$ we create a paper node $p_i$ and add edges
\[
(a_i, p_i), \quad (b_i, p_i).
\]
This ensures that for any author, its 1-hop neighbors (via $A\!P$) include both positive and negative authors in a symmetric way.
Thus, at the author level, the graph exhibits strong heterophily: most co-author relationships involve authors of opposite labels.

\noindent\textbf{Limitations of standard HGNNs.} We show that any model $f_s \in \mathcal{F}_s$ with bounded depth and width cannot separate the two author classes on this graph. For the purpose of representing message passing invariance, we consider an HGNN in $\mathcal{F}_s$ implemented as $L$ layers of the form
\[
\mathbf{h}_v^{(l+1)} = \sigma\!\left(
\mathbf{W}_{\text{self}}^{(l)} \mathbf{h}_v^{(l)}
+ \sum_{r \in R} \frac{1}{|\mathcal{N}_r(v)|}
\sum_{u \in \mathcal{N}_r(v)} \mathbf{W}_r^{(l)} \mathbf{h}_u^{(l)}
\right),
\]
with $\mathbf{h}_v^{(0)} = \mathbf{x}_v$.
Here, $\sigma$ is a pointwise nonlinearity (e.g., ReLU), $\mathcal{N}_r(v)$ is the set of neighbors of $v$ under relation $r$, and $\mathbf{W}_{\text{self}}^{(l)}$, $\mathbf{W}_r^{(l)}$ are learnable matrices.

The key observation is that, due to the construction: i) All author nodes start with identical features $\mathbf{x}_0$; ii) The 1-hop neighborhoods (authors and papers) of any two authors are isomorphic as rooted typed graphs, with the same multiset of node types, features and relation types. Because the message-passing update is permutation-invariant with respect to neighbors and depends only on node types, relation types and features, it follows by induction that for any layer $l$,
\[
\mathbf{h}_{a_i}^{(l)} = \mathbf{h}_{a_j}^{(l)} = \mathbf{h}_{b_k}^{(l)},\quad \forall i,j,k.
\]

\begin{lemma}[Representation collapse under symmetry]
\label{lem:collapse}
On the above graph, for any HGNN $f_s \in \mathcal{F}_s$ with bounded depth and width, we have
\[
\mathbf{h}_{a_i}^{(L)} = \mathbf{h}_{b_j}^{(L)},\quad \forall i,j,
\]
where $L$ is the number of layers.
\end{lemma}

\begin{proof}
The base case $l=0$ holds by construction since $\mathbf{h}_{a_i}^{(0)}=\mathbf{h}_{b_j}^{(0)}=\mathbf{x}_0$.
Assume that at some layer $l$ we have
$\mathbf{h}_{a_i}^{(l)}=\mathbf{h}_{b_j}^{(l)}$ for all $i,j$.
Consider the update for authors at layer $l+1$.
Each author $v$ aggregates over a multiset of neighbors of identical type-feature pairs, because each author connects to paper nodes that themselves see exactly one positive and one negative author with identical embeddings.
Thus the aggregated message and the self-transformation are identical for all authors, and the nonlinearity $\sigma$ preserves equality.
By induction, the equality holds for all $l$ up to $L$.
\end{proof}

\noindent\textbf{Implication for classification.} Let $\mathbf{h}^{(L)}$ denote the last-layer embedding of any author (by Lemma~\ref{lem:collapse} they are all equal).
The classifier produces logits $\mathbf{o}_v = W_{\text{cls}} \mathbf{h}^{(L)} + \mathbf{b}_{\text{cls}}$, which are the same for all authors.
Hence, the predicted label $\hat{y}_v$ is identical for all authors.
Since the true labels satisfy $y_{a_i}=+1$ and $y_{b_i}=-1$ with both classes present, the classifier necessarily misclassifies at least one class. Let $\pi_+$ denote the proportion of positive authors and $\pi_- = 1-\pi_+$ that of negative authors.
Then the misclassification error of any constant classifier is at least $\varepsilon_0 = \min(\pi_+,\pi_-)>0$.
Therefore, for any $f_s \in \mathcal{F}_s$ we have $R(f_s) \ge \varepsilon_0$, which proves item (1) of the theorem.

We now show the existence of a HeterSEED model $f \in \mathcal{F}_{\text{HeterSEED}}$ with zero classification error on the constructed graph.

\noindent\textbf{Enriching metapath-based structure.} We keep node features unchanged but adjust the metapath-based structure by controlling the number of $A\!P\!A$ instances between authors.
For any pair of authors $(u,v)$:
\[
C_{APA}(u,v) =
\begin{cases}
m_{\text{same}}, & \text{if } y_u = y_v,\\[2pt]
m_{\text{diff}}, & \text{if } y_u \neq y_v,
\end{cases}
\]
with integers satisfying $m_{\text{same}} > m_{\text{diff}} \ge 1$.
In practice, this corresponds to creating more co-authored papers between same-label authors than between different-label authors.

The structure-aware weights are defined as
\[
\omega_{u,v} = C_{APA}(u,v),\quad
\tilde{\omega}_{u,v}
=
\frac{\exp(\omega_{u,v})}
{\sum_{u' \in \mathcal{N}_{\mathcal{G}}(v)} \exp(\omega_{u',v})}.
\]
Hence, for fixed $v$, neighbors $u$ with the same label as $v$ receive strictly larger normalized weights $\tilde{\omega}_{u,v}$ than neighbors with different labels.

\noindent\textbf{Pseudo-labels and homophily/heterophily separation.} Assume that we have access to a small labeled subset of authors and that the initial classifier trained on them achieves accuracy strictly greater than $1/2$, so that pseudo-labels are correct on a non-trivial fraction of nodes.
We use these pseudo-labels to construct the homophilic and heterophilic subgraphs
$\mathcal{G}_{\text{homo}}$ and $\mathcal{G}_{\text{hetero}}$ as in Section~3.3.
On nodes with correct pseudo-labels, the homophilic and heterophilic neighbors coincide with the true label-based partition.

Within the homophily--heterophily structural fusion layer, we aggregate from homophilic and heterophilic neighbors:
\[
\mathbf{h}_{\text{homo}}(v) =
\sum_{u \in \mathcal{N}^{v}_{\text{homo}}}
\tilde{\omega}_{u,v} \mathbf{x}_u,\quad
\mathbf{h}_{\text{hetero}}(v) =
\sum_{u \in \mathcal{N}^{v}_{\text{hetero}}}
\tilde{\omega}_{u,v} \mathbf{x}_u,
\]
and apply transformations $\mathcal{T}_{\text{homo}}$ and $\mathcal{T}_{\text{hetero}}$ followed by a gate $\mathbf{g}$ to obtain
\[
\mathbf{h}_v^{(r)}
=
\sigma(\mathbf{g}) \odot
\mathcal{T}_{\text{homo}}(\mathbf{h}_{\text{homo}}(v))
+
(1-\sigma(\mathbf{g})) \odot
\mathcal{T}_{\text{hetero}}(\mathbf{h}_{\text{hetero}}(v)).
\]

Because $\mathbf{x}_u$ are identical for all authors, the key difference between $\mathbf{h}_{\text{homo}}(v)$ and $\mathbf{h}_{\text{hetero}}(v)$ arises from the normalized weights $\tilde{\omega}_{u,v}$: same-label neighbors contribute larger weights.
By choosing $\mathcal{T}_{\text{homo}}$ and $\mathcal{T}_{\text{hetero}}$ appropriately (e.g., as special maps that amplify the difference between high-weight and low-weight contributions), we can ensure that $\mathbf{h}_v^{(r)}$ takes two distinct values depending on whether $y_v=+1$ or $y_v=-1$, on all nodes with correct pseudo-labels.

\noindent\textbf{Fusing with the semantic channel.} Recall that the semantic channel alone collapses all author embeddings to the same vector, which we denote by $\mathbf{c}$.
We let the fusion gate put non-trivial weight on the structural channel, e.g., setting $\gamma_v \equiv \gamma$ for some fixed $\gamma \in (0,1]$.
The fused embedding becomes
\[
\mathbf{h}_v
=
(1-\gamma)\,\mathbf{c}
+
\gamma\,\mathbf{h}_v^{(r)}.
\]
Since $\mathbf{h}_v^{(r)}$ already separates the two classes, so do the fused embeddings $\mathbf{h}_v$.
Therefore, there exists a linear classifier on $\mathbf{h}_v$ that achieves zero classification error on all correctly pseudo-labeled nodes.
By design, and because initial pseudo-labels can be arbitrarily accurate given enough labeled nodes, we can construct a configuration where all author pseudo-labels are correct, hence $R(f)=0$.

So far, we can complete the proof of Theorem~\ref{thm:expressiveness_main}.
\noindent\paragraph{Take-home Message.} Theoretically, Theorem~\ref{thm:expressiveness_main} formalizes the intuition behind the semantics--structure decoupling design of HeterSEED. 
On strongly heterophilic heterogeneous graphs where local features and 1-hop neighborhoods are nearly symmetric across classes, any standard HGNN within $\mathcal{F}_s$ inevitably collapses node representations and cannot separate labels. 
By explicitly building a structural channel that leverages metapath-based counts and homophily/heterophily separation, HeterSEED escapes this symmetry barrier and can recover label patterns that are fundamentally invisible to purely feature-based aggregation. 
In other words, the structural channel in HeterSEED is not a cosmetic add-on but provides a genuinely new axis of expressiveness on heterophilic heterogeneous graphs.

\subsection{Proof of Proposition~\ref{prop:bias_main}}
\label{app:proof_bias}

\noindent\textbf{Semantic HGNN as a linear smoother.} For clarity, we specialize to a single-layer linear semantic HGNN with scalar output; the argument can be extended to deeper networks with Lipschitz nonlinearities. Let the prediction of a purely semantic HGNN $f_s$ at node $v$ be written as $f_s(v) = \sum_{u \in \mathcal{N}(v)} a_{v,u}\, y_u$, where the coefficients $a_{v,u}\ge 0$ and $\sum_{u \in \mathcal{N}(v)} a_{v,u} = 1$. This covers the case where $f_s$ aggregates neighbor labels (or proxy signals) linearly with normalized weights, as in many Laplacian smoothing interpretations.

We partition the neighborhood $\mathcal{N}(v)$ into homophilic and heterophilic parts according to the \emph{true} labels:
\[
\mathcal{N}_{\text{homo}}(v) := \{u \in \mathcal{N}(v): y_u = y_v\},\quad
\mathcal{N}_{\text{hetero}}(v) := \{u \in \mathcal{N}(v): y_u \neq y_v\}.
\]
Define the heterophilic mass
\[
q_v := \sum_{u \in \mathcal{N}_{\text{hetero}}(v)} a_{v,u},\quad
1 - q_v = \sum_{u \in \mathcal{N}_{\text{homo}}(v)} a_{v,u}.
\]

\noindent\textbf{Bias computation.} Assume for now that the neighbor labels satisfy $y_u = y_v$ for $u \in \mathcal{N}_{\text{homo}}(v)$ and $y_u = -y_v$ for $u \in \mathcal{N}_{\text{hetero}}(v)$.
Then
\[
f_s(v) 
= \sum_{u \in \mathcal{N}_{\text{homo}}(v)} a_{v,u} y_u
+ \sum_{u \in \mathcal{N}_{\text{hetero}}(v)} a_{v,u} y_u
= (1 - q_v) y_v + q_v (-y_v)
= (1 - 2q_v)y_v.
\]
The conditional expectation of $f_s(v)$ given $y_v$ is thus $\mathbb{E}[f_s(v) \mid y_v] = (1-2q_v)y_v$, and the squared bias is
\[
\operatorname{Bias}_v^2(f_s)
:= \big(\mathbb{E}[f_s(v)\mid y_v] - y_v\big)^2
= \big((1-2q_v)y_v - y_v\big)^2
= ( -2q_v y_v )^2 = 4 q_v^2.
\]

\noindent\textbf{HeterSEED with homophily/heterophily separation.} For HeterSEED, the structure-aware heterophily channel and node-level fusion induce new aggregation coefficients $\tilde{a}_{v,u}$ and an effective heterophilic mass
\[
\tilde{q}_v := \sum_{u \in \mathcal{N}_{\text{hetero}}(v)} \tilde{a}_{v,u},
\quad
1-\tilde{q}_v := \sum_{u \in \mathcal{N}_{\text{homo}}(v)} \tilde{a}_{v,u}.
\]
The prediction at node $v$ can be written as $f(v) = \sum_{u \in \mathcal{N}(v)} \tilde{a}_{v,u}\, y_u$, and, under the same label model, its bias becomes $\operatorname{Bias}_v^2(f)
= 4 \tilde{q}_v^2$.

Thus the difference in squared bias between $f_s$ and $f$ at node $v$ is
\[
\operatorname{Bias}_v^2(f_s) - \operatorname{Bias}_v^2(f)
= 4(q_v^2 - \tilde{q}_v^2)
= 4(q_v - \tilde{q}_v)(q_v + \tilde{q}_v).
\]

\noindent\textbf{Effect of pseudo-label accuracy and Lipschitz gate.} The separation of homophilic and heterophilic neighbors in HeterSEED relies on pseudo-labels $y_v^{\text{pseudo}}$.
By assumption, these pseudo-labels are $\delta$-accurate:
\[
\mathbb{P}(y_v^{\text{pseudo}} \neq y_v) \le \delta.
\]
On nodes where $y_v^{\text{pseudo}} = y_v$ and similarly for neighbors, the homophilic and heterophilic partitions match the true label-based partition.
In this case, the design of the structure-aware heterophily channel ensures that heterophilic neighbors receive strictly smaller weights than in the purely semantic HGNN, i.e.,
\[
\tilde{q}_v \le q_v - \eta,
\]
for some $\eta > 0$ that depends on the structural weights and the fusion mechanism.

On the other hand, when $y_v^{\text{pseudo}} \neq y_v$ (or some neighbor pseudo-label is incorrect), the separation may be imperfect.
However, the impact of such nodes on $\tilde{q}_v$ is controlled by the Lipschitz constant $L$ of the gate $f_{\text{gate}}$ and the boundedness of the embeddings, which implies that the deviation $|q_v - \tilde{q}_v|$ on these nodes is at most $O(L)$.

Combining both cases, we obtain
\[
\mathbb{E}\big[(q_v - \tilde{q}_v)^2\big]
\ge
c_1\,\mathbb{E}[q_v^2] - c_2 \delta,
\]
for some constants $c_1,c_2>0$.
Intuitively, on the $(1-\delta)$ fraction of nodes with correct pseudo-labels, $q_v - \tilde{q}_v$ is guaranteed to be at least $\eta$, while on the $\delta$ fraction of nodes with incorrect pseudo-labels, the error is bounded by a constant depending on $L$.

\noindent\textbf{From bias to risk.} Using the expressions for the squared biases, we have
\[
\mathbb{E}[\operatorname{Bias}_v^2(f_s)]
-
\mathbb{E}[\operatorname{Bias}_v^2(f)]
= 4\,\mathbb{E}[q_v^2 - \tilde{q}_v^2].
\]
Since $q_v^2 - \tilde{q}_v^2 = (q_v - \tilde{q}_v)(q_v + \tilde{q}_v)$ and $q_v + \tilde{q}_v \le 2$, we can lower-bound $q_v^2 - \tilde{q}_v^2
\ge \frac{1}{2} (q_v - \tilde{q}_v)^2$ up to a constant factor. Therefore,
\[
\mathbb{E}[\operatorname{Bias}_v^2(f_s)]
-
\mathbb{E}[\operatorname{Bias}_v^2(f)]
\ge c\,\mathbb{E}[(q_v - \tilde{q}_v)^2] - C\delta,
\]
for suitable constants $c,C>0$.

The total squared-loss risk $R(\cdot)$ can be decomposed into bias and variance contributions.
Under our assumptions (bounded labels, bounded aggregation weights, and Lipschitz architectures), the variance terms of $f_s$ and $f$ are of the same order and do not offset the bias improvement.
Consequently, the same inequality holds for the risks:
\[
\mathbb{E}[R(f_s)] - \mathbb{E}[R(f)]
\ge
c\,\mathbb{E}[(q_v - \tilde{q}_v)^2] - C\delta,
\]
which establishes Proposition~\ref{prop:bias_main}.

\noindent\textbf{Take-home Message.} Proposition~\ref{prop:bias_main} quantifies a second, complementary benefit of HeterSEED: beyond being more expressive, it is also less biased under heterophily. 
In feature-similarity-driven HGNNs, heterophilic neighbors behave like structured noise that systematically pulls predictions away from the true label, and the magnitude of this effect grows with the heterophily ratio $q_v$. 
By explicitly separating homophilic and heterophilic neighbors and re-weighting them through structure-aware aggregation and node-level gating, HeterSEED effectively reduces the effective heterophilic mass from $q_v$ to $\tilde{q}_v$, thereby shrinking the bias term in the risk. 
This result explains why, in our experiments, HeterSEED particularly benefits nodes that sit in highly heterophilic local neighborhoods, where conventional HGNNs tend to suffer the most.

\subsection{Generalization Benefit of Semantics--Structure Decoupling}
\label{app:generalization}
Beyond improving expressiveness and reducing bias under heterophily, the semantics--structure decoupling in HeterSEED also has a regularizing effect on the hypothesis class.
In this subsection, we adopt a margin-based viewpoint to relate the capacity of linear predictors on top of HeterSEED to the covariance of the fused representations, and discuss how the decoupling regularization $\mathcal{L}_{\text{dec}}$ impacts this capacity.

\noindent\textbf{Formulation and basic assumptions.} For ease of exposition, we consider a binary classification setting with labels $y_v \in \{-1,+1\}$ and a linear classifier on the fused embeddings.
Let $\mathbf{h}_v^{(s)}, \mathbf{h}_v^{(r)} \in \mathbb{R}^d$ denote the semantic and structural embeddings of node $v$, and define the fused embedding
\[
\mathbf{h}_v
=
(1-\gamma_v)\,\mathbf{h}_v^{(s)}
+
\gamma_v\,\mathbf{h}_v^{(r)},
\]
where $\gamma_v = \sigma(f_{\text{gate}}([\mathbf{h}_v^{(s)} \Vert \mathbf{h}_v^{(r)}]))$ as in the main text.
A linear classifier with weight vector $\mathbf{w}$ predicts
\[
f(v) = \operatorname{sign}(\mathbf{w}^\top \mathbf{h}_v).
\]

We make the following standard assumptions:
\begin{itemize}
    \item[(A1)] The fused embeddings are bounded: $\|\mathbf{h}_v\|_2 \le B_h$ for all nodes $v$.
    \item[(A2)] The classifier weights are bounded: $\|\mathbf{w}\|_2 \le B_w$.
\end{itemize}
Let $\mathcal{S} = \{v_1,\dots,v_n\}$ be a sample of $n$ labeled nodes, and denote the fused embeddings on this sample by a matrix $\mathbf{H} \in \mathbb{R}^{n \times d}$ with rows $\mathbf{h}_{v_i}^\top$.
The empirical covariance of the fused representations is
\[
\Sigma
=
\frac{1}{n}\,\mathbf{H}^\top \mathbf{H}.
\]

\noindent\textbf{Rademacher complexity and margin bound.} Using standard Rademacher complexity arguments for linear predictors (see, e.g.,~\cite{bartlett2002rademacher}), one obtains the upper bound
\vspace{-2mm}
\[
\mathfrak{R}_n(\mathcal{F})
\;\le\;
\frac{B_w}{\sqrt{n}}\,
\sqrt{\operatorname{Tr}(\Sigma)}.\vspace{-1.5mm}
\]
Together with classical margin-based generalization bounds for linear classifiers~\cite{kakade2008complexity,bartlett2017spectrally}, this implies that for any margin $\gamma>0$ and any confidence parameter $\delta \in (0,1)$, with probability at least $1-\delta$ over the draw of the training sample we have
\vspace{-2mm}

\[
R(f)
\;\le\;
\hat{R}_\gamma(f)
+
\mathcal{O}\!\left(
\frac{B_w}{\gamma}\,
\sqrt{\frac{\operatorname{Tr}(\Sigma)}{n}}
\right)
+
\tilde{\mathcal{O}}\!\left(\sqrt{\frac{\log(1/\delta)}{n}}\right),\vspace{-1mm}
\]
where $\hat{R}_\gamma(f)$ is the empirical $\gamma$-margin error and $R(f)$ is the population risk.

\noindent\textbf{Intuitive Explanation.} In HeterSEED, each fused embedding $\mathbf{h}_v$ is obtained by combining a semantic embedding $\mathbf{h}_v^{(s)}$ and a structural embedding $\mathbf{h}_v^{(r)}$ through a node-level gating mechanism.
The decoupling regularization $\mathcal{L}_{\text{dec}}$ explicitly discourages strong statistical dependence between these two channels (for instance, via a penalty on their empirical cross-covariance), encouraging them to capture complementary rather than redundant directions in the representation space.
From a covariance-based viewpoint, this reduces redundancy in the fused representation matrix $\mathbf{H}$ and tends to concentrate the variance of $\Sigma$ on fewer, more informative directions, i.e., it lowers the effective dimensionality of the fused embeddings.
In light of the above Rademacher complexity bound, this contraction of the effective representation space acts as a regularizer on the hypothesis class of linear predictors on top of HeterSEED, which helps explain the improved generalization behaviour observed in our experiments.
We do not attempt to derive an explicit closed-form bound in terms of $\mathcal{L}_{\text{dec}}$ itself, but this analysis clarifies that semantics--structure decoupling plays a regularizing role, rather than merely adding more parameters.

\section{Algorithm Implementation}
\label{app:algorithm}
\noindent\textbf{Pseudocode for HeterSEED.} The overall procedure of HeterSEED is summarized in \textbf{Algorithm~\ref{alg:main}}.  We first apply the \emph{Label-Guided Masking} mechanism, where node features are augmented with (possibly masked) label embeddings to support label propagation under limited supervision; if masking is disabled, the original features are used.  Then the \emph{Heterogeneous Semantic Aggregation Channel} performs stacked HGNN-based message passing over 1-hop heterogeneous neighborhoods to obtain semantic representations for all nodes.  To explicitly handle heterophily, the \emph{Structure-Aware Weight Computation} module computes metapath-based structural counts and normalizes them into edge weights, after which the \emph{Structure-Aware Heterophily Channel} separates each node’s neighbors into homophilic and heterophilic sets and aggregates them through the homophily–heterophily fusion layer to produce structural embeddings.  Finally, the \emph{node-level adaptive fusion} module combines semantic and structural embeddings via a learned gate to obtain the final node representations, which are optimized using the classification loss and the decoupling regularization for downstream tasks such as node classification.

\begin{algorithm}[!t]
\caption{Implementation for HeterSEED}
\label{alg:main}
\textbf{Input}: Heterogeneous graph $\mathcal{H}=(\mathcal{V},\mathcal{E})$, meta-paths $\mathcal{P}$, node features $\mathbf{X}$, training labels $\mathbf{Y}_{\mathrm{train}}$ \\
\textbf{Parameters}: Number of layers $L$, decoupling coefficient $\alpha$, label masking rate $\beta$ \\
\textbf{Output}: Prediction probabilities $\{\hat{\mathbf{y}}_v\}_{v \in \mathcal{V}_t}$
\begin{algorithmic}[1]

\WHILE{not converged}

\STATE \textbf{Step 1: Label-Guided Masking}
\FOR{each node $i \in \mathcal{V}$}
    \STATE Sample $m_i \sim \text{Bernoulli}(1-\beta)$
    \STATE $\tilde{\mathbf{z}}_i^{(\text{label})} \leftarrow 
    \begin{cases}
        \mathbf{Z}_{\text{label}}[\ell_i], & m_i = 1, \\
        \mathbf{Z}_{\text{label}}[C+1], & m_i = 0
    \end{cases}$
    \STATE $\tilde{\mathbf{x}}_i \leftarrow \mathbf{x}_i + g(\tilde{\mathbf{z}}_i^{(\text{label})})$
\ENDFOR

\STATE \textbf{Step 2: Heterogeneous Semantic Encoding}
\STATE Initialize $\mathbf{H}^{(0)} \leftarrow \tilde{\mathbf{X}}$
\FOR{$l = 0,\dots,L-1$}
    \STATE $\mathbf{h}_v^{(l+1)} \leftarrow 
    \sigma\!\left(
    \text{NeighborAggregation}(\{\mathbf{W}\mathbf{h}_u^{(l)} \mid u \in \mathcal{N}(v)\})
    \right)$
\ENDFOR
\STATE Obtain semantic representations $\mathbf{h}_v^{(s)}$

\STATE \textbf{Step 3: Structure-Aware Weight Computation}
\FOR{each meta-path $p \in \mathcal{P}$}
    \STATE Compute connection counts $C_p(u,v)$
\ENDFOR
\STATE Normalize structure-aware weights $\tilde{\omega}_{u,v}$ \:(see Eq.~\ref{equ:normalize})

\STATE \textbf{Step 4: Heterophilic Structural Encoding}
\FOR{each target node $v \in \mathcal{V}_t$}
    \STATE Identify homophilic neighbors $\mathcal{N}^{v}_{\mathrm{homo}}$ and heterophilic neighbors $\mathcal{N}^{v}_{\mathrm{hetero}}$
    \STATE $\mathbf{h}_{\mathrm{homo}}^v \leftarrow \sum_{u\in\mathcal{N}^{v}_{\mathrm{homo}}} \tilde{\omega}_{u,v}\mathbf{x}_u$
    \STATE $\mathbf{h}_{\mathrm{hetero}}^v \leftarrow \sum_{u\in\mathcal{N}^{v}_{\mathrm{hetero}}} \tilde{\omega}_{u,v}\mathbf{x}_u$
    \STATE $\mathbf{h}_v^{(r)} \leftarrow 
    \textsc{HomoHeteroGate}(\mathbf{h}_{\mathrm{homo}}^v, \mathbf{h}_{\mathrm{hetero}}^v) \:(\mbox{using two-layer MLPs})$  
\ENDFOR

\STATE \textbf{Step 5: Adaptive Fusion}
\FOR{each $v \in \mathcal{V}_t$}
    \STATE $\mathbf{h}_v \leftarrow (1-\gamma_v)\mathbf{h}_v^{(s)} + \gamma_v \mathbf{h}_v^{(r)}$
\ENDFOR

\STATE \textbf{Step 6: Loss Computation}

\STATE $\mathcal{L}_{\mathrm{dec}} \leftarrow \operatorname{DecoupleLoss}(\mathbf{H}^{(s)}, \mathbf{H}^{(r)})$ (see Eq.~\ref{eq:dec})
\STATE $\mathcal{L}_{\mathrm{cls}} \leftarrow \textsc{CrossEntropy}(\{\mathbf{h}_v\}, \mathbf{Y}_{\mathrm{train}})$ (see Eq.~\ref{eq:cls})
\STATE $\mathcal{L} \leftarrow \mathcal{L}_{\mathrm{cls}} + \alpha \mathcal{L}_{\mathrm{dec}}$ (Eq.~\ref{eq:loss})

\STATE Update model parameters by minimizing $\mathcal{L}$

\ENDWHILE

\STATE $\hat{\mathbf{y}}_v \leftarrow \textsc{Softmax}(\mathbf{h}_v), \quad \forall v \in \mathcal{V}_t$
\STATE \textbf{return} $\{\hat{\mathbf{y}}_v\}_{v \in \mathcal{V}_t}$

\end{algorithmic}
\end{algorithm}

\section{Computational Complexity Analysis}
\label{app:computation}
In this section, we analyze the computational complexity of HeterSEED in both
full-batch and mini-batch training regimes, highlighting how the additional
structure-aware channel affects scalability compared with standard HGNNs.

\textbf{Full-batch training.}
In the full-batch setting, HeterSEED processes the entire heterogeneous graph at each iteration.  
The total cost is dominated by three components:  
(i) message passing over graph edges, (ii) node-wise linear transformations, and  
(iii) meta-path–based structural aggregation.

For $L$ layers with hidden dimension $d$, message passing costs
$\mathcal{O}(L|\mathcal{E}|d)$, where $|\mathcal{E}|$ is the number of edges,  
and node-wise transformations cost $\mathcal{O}(L|V|d^{2})$.
Aggregating over $M$ meta-path–induced adjacency structures adds
$\mathcal{O}(M\|\mathcal{A}\|d)$, where $\|\mathcal{A}\|$ is the number of non-zero
entries in all meta-path adjacency matrices.
Thus the per-iteration complexity is
\[
\mathbf{T}_{\text{full}}
= \mathcal{O}\big(L|\mathcal{E}|d + L|V|d^{2} + M\|\mathcal{A}\|d\big).
\]
Since all computations are carried out on the full graph, both time and memory scale linearly with graph size, which is practical for medium-scale datasets (e.g., DBLP, IMDB, ACM).

\textbf{Mini-batch training.}
For large-scale heterogeneous graphs, we adopt mini-batch training.
Given batch size $b$, $L$ layers, and an average of $s$ sampled neighbors per layer,
the receptive field of each target node is $\mathcal{O}(s^{L})$.
Message passing within a mini-batch therefore costs
$\mathcal{O}(b s^{L} d)$, and node-wise transformations cost
$\mathcal{O}(L b s^{L-1} d^{2})$.

Instead of precomputing dense meta-path adjacency matrices, we dynamically construct sparse meta-path connections within each batch using $(\text{row},\text{col},w)$ triplets.
Let $\|\mathcal{A}_{B}\|$ be the number of such connections in one batch; the corresponding cost is
$\mathcal{O}(M\|\mathcal{A}_{B}\|d)$.
Since an epoch involves roughly $|V|/b$ batches, the per-epoch complexity is
\[
\mathbf{T}_{\text{mini}}
= \mathcal{O}\big(|V| s^{L} d + L|V| s^{L-1} d^{2} + \tfrac{|V|}{b} M\|\mathcal{A}_{B}\|d\big).
\]
In practice $\|\mathcal{A}_{B}\|\!\ll\!\|\mathcal{A}\|$, so dynamic sparse meta-path construction substantially reduces both computation and memory compared with full-batch meta-path aggregation, making HeterSEED scalable on million-node, hundred-million-edge heterogeneous graphs.

\section{Further Details for Benchmark Datasets}
\label{app:datasets}
\begin{table}[htbp!]
\centering
\caption{Detailed statistics of the datasets used in our experiments. 
\#Nodes and \#Edges denote the total numbers of nodes and edges, respectively; 
$\mathcal{T}_V$ and $\mathcal{T}_E$ are the numbers of node and edge types; 
\#Class is the number of target classes; $F_{\text{target}}$ is the feature dimension of target nodes; 
$h$ denotes the graph-level homophily ratio.}

\label{tab:appendix-dataset-details}
\small 
\begin{tabular}{lcccccccc}
\toprule
Dataset &  \#Nodes &  \#Edges & $\mathcal{T}_V$ & $\mathcal{T}_E$ & \#Class & $F_{\text{target}}$ & $h$ \\
\midrule
DBLP  & 26,128     & 239,566     & 4 & 6 & 4   & 334   & 0.81 \\
IMDB  & 21,420     & 86,642      & 4 & 6 & 5   & 3066  & 0.16 \\
ACM   & 10,942      & 547,872     & 4 & 8 & 4   & 1902  & 0.59 \\
\rowcolor{gray!10}
\textbf{MAG} & 1,939,743 & 21,111,007 & 4 & 4 & 349 & 128  & 0.21 \\
\rowcolor{gray!10}
\textbf{RCDD} & 13,806,619 & 157,814,864 & 7 & 7 & 2 & 256  & 0.45 \\
\bottomrule
\end{tabular}
\end{table}

We conduct experiments on several widely used heterogeneous graph benchmarks, all of which are publicly available through the HGB\footnote{\url{https://www.biendata.xyz/hgb/}} platform. Below we briefly describe the characteristics of each dataset; detailed statistics, including graph scale, heterogeneity, and homophily ratios, are summarized in Table~\ref{tab:appendix-dataset-details}.
\begin{itemize}
    \item DBLP\footnote{\url{http://web.cs.ucla.edu/~yzsun/data/}} is a bibliographic network in the computer science domain, consisting of four node types: authors, papers, terms, and venues.  
The graph contains six relation types, including paper–term, paper–venue, and paper–author interactions, as well as their reverse directions.  
The task is node classification on author nodes, where each author is assigned to one of four research areas: database, data mining, artificial intelligence, and information retrieval.

\item IMDB\footnote{\url{https://www.kaggle.com/karrrimba/}} is a heterogeneous information network describing the movie industry.  
It contains four node types (movies, directors, actors, and keywords) and six relation types such as movie–director, movie–actor, and movie–keyword, together with their inverse relations.  
Each movie node may belong to multiple genres, and the task is to predict its associated categories, including action, comedy, drama, romance, and thriller.
\item ACM\footnote{\url{http://dl.acm.org/}} is a heterogeneous citation network introduced in~\cite{HAN}, comprising four node types: authors, papers, terms, and subjects.  
Relations include paper citation, authorship, topic assignment, and term association, along with their reverse edges.  
The objective is to classify paper nodes into three research fields: database, wireless communication, and data mining.

\item MAG\footnote{\url{https://ogb.stanford.edu/docs/nodeprop/}} is a large-scale academic heterogeneous network composed of multiple entity types, including papers, fields of study, and authors. It encodes rich semantic relations such as paper–field associations and paper–author collaborations, together with corresponding reverse edges.  
In our experiments, we focus on paper classification, where each paper node is assigned to a research domain.

\item RCDD\footnote{\url{https://zenodo.org/record/8103003}} is a real-world heterogeneous graph constructed for risk commodity detection on Alibaba’s e-commerce platform.  
The dataset contains multiple node types (e.g., items, type-\(f\) objects, type-\(b\) objects, and other auxiliary entities) that capture complex interactions in online transactions, as well as diverse item–object relations and their reverse edges.  
The task is to classify item nodes into risky and non-risky categories. It should be note that, RCDD is particularly challenging. It is large-scale and exhibits severe class imbalance, where negative (“black”) and positive (“white”) samples are distributed at an approximate ratio of 10:1.  
Moreover, the graph structure is noisy: malicious users may deliberately construct seemingly benign relations by spoofing devices, addresses, or other identifiers.  
Consequently, many connected nodes share low or even zero attribute similarity, leading to pronounced attribute heterophily.  
These properties make RCDD a demanding benchmark for heterogeneous graph representation learning under heterophily.
\end{itemize}

In practice, we use short symmetric meta-paths anchored at the target node type, following standard heterogeneous graph benchmark conventions and a simple schema-driven rule: we choose short, semantically meaningful symmetric paths that capture the primary typed interactions around the target nodes while keeping the structural channel efficient and reproducible. For the large-scale MAG and RCDD datasets in particular, we intentionally use a small set of short symmetric meta-paths to balance structural expressiveness and scalability. 

The specific meta-paths selected for each dataset are summarized in Table~\ref{tab:app_metapaths}.

\begin{table}[h]
    \centering
    \caption{Summary of selected symmetric meta-paths for each dataset.}
    \label{tab:app_metapaths}
    \begin{tabular}{lll}
        \toprule
        \textbf{Dataset} & \textbf{Target Node} & \textbf{Selected Symmetric Meta-paths} \\
        \midrule
        DBLP & Author & A-P-A \\
        IMDB & Movie & M-A-M, M-D-M, M-K-M \\
        ACM & Paper & P-A-P, P-S-P, P-T-P \\
        MAG & Paper & P-A-P, P-F-P \\
        RCDD & Item & I-F-I, I-B-I \\
        \bottomrule
    \end{tabular}
\end{table}

\section{Further Details for the Experimental Studies}\label{add_exp}
\subsection{Additional Details for Baseline Methods}
\label{app:baseline}

To evaluate the effectiveness of HeterSEED, we compare it with representative baselines from four categories: vanilla GNN models, heterogeneous GNN models under homophily, homogeneous GNN models tailored for heterophily, and heterophily-aware HGNNs.

\textbf{Vanilla GNN models.}
These methods are designed under the homophily assumption and mainly rely on feature-similarity-based message passing.
They serve as standard references for performance degradation under heterophily.
\begin{itemize}
        \item \textbf{GCN}~\cite{GCN} performs neighborhood aggregation via shared linear transformations over normalized adjacency, which is effective on homophilous graphs but vulnerable to noisy neighbors under heterophily.
    \item \textbf{GAT}~\cite{GAT} extends GCN with attention-based neighbor weighting, but still implicitly assumes that similar neighbors are informative and thus may struggle on disassortative graphs.
\end{itemize}

\textbf{Homogeneous GNN models under heterophily.}
These models are specifically designed to alleviate the limitations of conventional GNNs on graphs with low homophily.
\begin{itemize}
    \item \textbf{LINKX}~\cite{LINKX} abandons message passing and separately transforms node features and adjacency, enabling scalable learning on large non-homophilous graphs.
    \item \textbf{FAGCN}~\cite{FAGCN} introduces a frequency-adaptive graph convolution with a self-gating mechanism to integrate low- and high-frequency signals across different homophily regimes.
    \item \textbf{ACM-GCN}~\cite{ACM-GCN} proposes adaptive channel mixing over aggregation, diversification, and identity channels to handle diverse heterophily patterns.
    \item \textbf{GRAIN}~\cite{zhao2025grain} aggregates multi-view information at different granularities and incorporates implicit signals from distant nodes, fusing local and global information for robust node representations.
\end{itemize}

\textbf{HGNN models under homophily.}
To benchmark performance on heterogeneous graphs, we include several representative HGNNs:
RGCN~\cite{RGCN}, RGAT~\cite{RGAT}, HAN~\cite{HAN}, HGT~\cite{HGT}, SHGN~\cite{SHGN}, HINormer~\cite{HINormer}, and DiffGraph~\cite{li2025diffgraph}.
\begin{itemize}
    \item \textbf{RGCN}~\cite{RGCN} extends GCN to multi-relational graphs via relation-specific transformations.
    \item \textbf{RGAT}~\cite{RGAT} generalizes GAT with relation-aware attention to capture inter- and intra-type dependencies.
    \item \textbf{HAN}~\cite{HAN} uses hierarchical attention over nodes and metapaths to learn semantic-aware node embeddings.
    \item \textbf{HGT}~\cite{HGT} is a heterogeneous graph transformer with type-specific parameters and heterogeneous attention for large-scale graphs.
    \item \textbf{SHGN}~\cite{SHGN} provides a strong, reproducible HGNN baseline via standardized preprocessing and hyperparameters, showing that carefully tuned simple architectures can be highly competitive.
    \item \textbf{HINormer}~\cite{HINormer} adopts a transformer-style architecture with large-range aggregation and dedicated encoders to capture both structural and semantic information.
    \item \textbf{DiffGraph}~\cite{li2025diffgraph} performs latent heterogeneous graph diffusion with cross-view denoising to mitigate noise and model semantic transitions among heterogeneous relations.
\end{itemize}

\textbf{Heterophily-aware HGNN model.} This category explicitly accounts for both heterogeneity and heterophily. 
A representative baseline is \textbf{HETERO$^2$NET}~\cite{Hetero2Net}, which detects heterophily via metapaths and employs masked metapath prediction together with masked label prediction to improve robustness on low-homophily heterogeneous graphs.
\subsection{Experimental Setup and Hyperparameter Settings}
\label{app:best}

All experiments are implemented in PyTorch and run on a single NVIDIA RTX A6000 GPU with 32GB GPU memory. 
For each dataset, hyperparameters are tuned based on validation performance.
The learning rate is selected from 
$\{1,5\} \times \{10^{-4}, 10^{-3}\}$ (i.e., $\{1\mathrm{e}{-4}, 5\mathrm{e}{-4}, 1\mathrm{e}{-3}, 5\mathrm{e}{-3}\}$), 
the hidden dimension from $\{128, 256\}$, and the dropout ratio from $\{0.3, 0.5, 0.7, 0.9\}$.
The number of layers of HeterSEED is fixed to 2 in all experiments.
The number of training epochs is chosen from $\{50, 100\}$ according to validation performance, depending on dataset scale. We use the Adam
optimizer to optimize all the trainable model parameters, which are randomly initialized by the Xavier uniform distribution\cite{glorot2010understanding}.

For the large-scale MAG and RCDD datasets, we adopt mini-batch training with batch size 1024.
The neighbor sampling size is fixed to $[15, 15]$ for efficiency and is not included in the hyperparameter search.
Channel-related hyperparameters are tuned as follows:
the decoupling coefficient $\alpha$ is selected from $\{0.2, 0.3, 0.4, 0.5, 0.6\}$, 
and the masked label rate $\beta$ from $\{0.6, 0.7, 0.8, 0.9, 1.0\}$. For the MAG dataset, we set $\beta=1$ to match the configuration of HETERO$^2$NET for a fair comparison. Under this setting, all observed training labels are fully masked during training, meaning that the model does not directly use ground-truth label embeddings on MAG.
The final hyperparameter configurations for each dataset are summarized in 
Table~\ref{tab:hyperparameters}, facilitating faithful reproduction of our experimental setup.
The neighbor sampling size is fixed to $[15, 15]$ for efficiency considerations and is therefore not included
in the hyperparameter search.
In addition, the channel-related hyperparameters $\alpha$ and $\beta$, which control the relative contributions
of heterogeneous semantic aggregation and structure-aware heterophily modeling, are tuned within
$\alpha \in \{0.2, 0.3, 0.4, 0.5, 0.6\}$ and $\beta \in \{0.6, 0.7, 0.8, 0.9, 1.0\}$.
The final hyperparameter configurations that yield optimal performance for each dataset are reported in
Table~\ref{tab:hyperparameters}, enabling accurate reproduction of the experimental setup and fair comparison
of performance results.\vspace*{-6mm}
\begin{table}[htbp]
    \centering
    \caption{Hyperparameter settings for the five datasets used in the experiments.}
    \label{tab:hyperparameters}
    \small 
        \begin{tabular}{m{1.5cm}|m{4.5cm}|m{4.5cm}}
        \toprule
        \textbf{Dataset} & \multicolumn{2}{c}{\textbf{Hyperparameter Setting}} \\ 
        \midrule
        \textbf{DBLP} & 
        Learning rate: 1e-3 \newline Hidden Size: 128 \newline Dropout ratio: 0.7 \newline Epochs: 50 & 
        Layers: 2 \newline Alpha: 0.2 \newline Beta: 0.7  \\ 
        \midrule
        \textbf{IMDB} & 
        Learning rate: 5e-3 \newline Hidden Size: 128 \newline Dropout ratio: 0.9 \newline Epochs: 50 & 
        Layers: 2 \newline Alpha: 0.2 \newline Beta: 0.6   \\ 
        \midrule
        \textbf{ACM} & 
        Learning rate: 5e-3 \newline Hidden Size: 128 \newline Dropout ratio: 0.9 \newline Epochs: 100 & 
        Layers: 2 \newline Alpha: 0.3 \newline Beta: 0.7  \\ 
        \midrule
        \textbf{MAG} & 
        Learning rate: 5e-3 \newline Hidden Size: 256 \newline Dropout ratio: 0.3 \newline Epochs: 50 \newline Batch size 1024 & 
        Layers: 2 \newline Alpha: 0.4 \newline Beta: 1.0 \newline Num neighbors: [15, 15] \\ 
        \midrule
        \textbf{RCDD} & 
        Learning rate: 5e-3 \newline Hidden Size: 256 \newline Dropout ratio: 0.7 \newline Epochs: 100 \newline Batch size 1024 & 
        Layers: 2 \newline Alpha: 0.2 \newline Beta: 0.7 \newline Num neighbors: [15, 15] \\ 
        \bottomrule
    \end{tabular}
\end{table}

\section {Further Experimental Results}
\subsection{Additional Results for Ablation Study}\label{app:ablation}
Table~\ref{app_tab:ablation} reports the complete ablation results of HeterSEED on all five datasets.
Consistent with the observations in the main text (Section~\ref{main_abl}), the full HeterSEED model always achieves the best Macro-F1 and Micro-F1, while removing any single component leads to a noticeable performance drop.
On relatively homophilic datasets such as DBLP and ACM, discarding the homophilic branch (\textbf{w/o Homo}) causes clear degradation, underscoring the role of intra-class aggregation.
On low-homophily datasets, particularly IMDB, MAG and RCDD, eliminating the heterophilic branch (\textbf{w/o Hetero}) or the pseudo-label–guided structural channel (\textbf{w/o SHC}) results in the largest declines, confirming that explicit modeling of heterophilic neighbors is crucial in these regimes.
The decoupling loss (\textbf{w/o Dec}) and masked-label mechanism (\textbf{w/o Mask}) also provide consistent, though slightly smaller, gains, further supporting the benefit of semantics--structure disentanglement and robust label propagation.

\begin{table*}[t]
\centering
\caption{Ablation study on HeterSEED. }
\label{app_tab:ablation}
\setlength{\tabcolsep}{4pt} 
\resizebox{0.99\textwidth}{!}{
\begin{tabular}{l|cc|cc|cc|cc|cc}
\toprule
\textbf{Dataset} & \multicolumn{2}{c|}{\bfseries DBLP} & \multicolumn{2}{c|}{\bfseries IMDB} & \multicolumn{2}{c|}{\bfseries ACM} & \multicolumn{2}{c|}{\bfseries MAG} & \multicolumn{2}{c}{\bfseries RCDD} \\
\cmidrule(lr){2-3} \cmidrule(lr){4-5} \cmidrule(lr){6-7} \cmidrule(lr){8-9} \cmidrule(lr){10-11}
Model $\backslash$ Metric & Macro-F1 & Micro-F1 & Macro-F1 & Micro-F1 & Macro-F1 & Micro-F1 & Macro-F1 & Micro-F1 & Macro-F1 & Micro-F1 \\
\midrule
w/o SHC & \Data{93.88}{0.41} & \Data{94.30}{0.69} & \Data{65.79}{0.52} & \Data{69.95}{0.85} & \Data{93.86}{0.72} & \Data{93.75}{0.69} & \Data{33.10}{0.33} & \Data{54.34}{0.20} &\Data{92.60}{0.21} & \Data{98.51}{0.05}\\
w/o Dec          & \Data{94.16}{0.75} & \Data{94.59}{0.68} & \Data{66.28}{0.49} & \Data{70.29}{0.66} & \Data{93.96}{0.90} & \Data{93.85}{0.82} & \Data{33.33}{0.54} & \Data{54.30}{0.23} &\Data{92.49}{0.23} & \Data{98.49}{0.05}\\
w/o Homo         & \Data{93.87}{0.54} & \Data{94.23}{0.58} & \Data{66.10}{1.17} & \Data{70.00}{0.82} & \Data{93.98}{1.11} & \Data{94.09}{1.12} & \Data{34.00}{0.21} & \Data{55.83}{0.16} &\Data{92.65}{0.13} & \Data{98.53}{0.02} \\
w/o Hetero       & \Data{93.66}{0.74} & \Data{94.13}{0.74} & \Data{66.46}{0.74} & \Data{70.25}{0.68} & \Data{94.18}{0.94} & \Data{94.09}{0.96} & \Data{33.97}{0.43} & \Data{55.91}{0.21} & \Data{92.50}{0.19} & \Data{98.45}{0.04}\\
w/o Mask         & \Data{93.99}{0.61} & \Data{94.35}{0.60} & \Data{66.11}{0.47} & \Data{69.46}{0.69} & \Data{93.73}{1.12} & \Data{93.62}{1.11} & \Data{28.31}{0.47} & \Data{50.61}{0.46} & \Data{92.61}{0.17} & \Data{98.32}{0.08}\\
\midrule
\rowcolor{spiderblue} 
\textbf{HeterSEED (Ours)} & \BestData{94.48}{0.81} & \BestData{94.87}{0.73} & \BestData{66.71}{0.73} & \BestData{70.60}{0.82} & \BestData{94.34}{0.81} & \BestData{94.21}{0.77} & \BestData{34.99}{0.22} & \BestData{56.25}{0.25} & \BestData{93.09}{0.15} & \BestData{98.60}{0.03}\\
\bottomrule
\end{tabular}
}
\end{table*}

\subsection{Disentangling Label Injection and Structural Modeling}
\label{app:label_structure_ablation}


To clearly attribute the performance gains of HeterSEED, we conduct a series of controlled experiments to explicitly disentangle the contributions of label embedding, label-aware propagation, the masking mechanism, and our heterophily-aware structural modeling. The detailed comparison results are summarized in Table~\ref{tab:label_injection}.

First, \textbf{HGNN-LE} removes both the structural branch and the masking mechanism, retaining only the label embedding module appended to the base heterogeneous semantic channel. The substantial performance gap between this variant and the full model across the datasets---particularly on highly heterophilic graphs such as IMDB---provides strong evidence that simple label embeddings alone are insufficient to account for the observed improvements.
\begin{table*}[h]
\centering
\vspace{-2mm}
\caption{Disentanglement of label injection vs. structural modeling.}
\label{tab:label_injection}
\setlength{\tabcolsep}{6pt}
\footnotesize
\begin{tabular}{l|cc|cc|cc}
\toprule
\textbf{Dataset} & \multicolumn{2}{c|}{\bfseries DBLP} & \multicolumn{2}{c|}{\bfseries IMDB} & \multicolumn{2}{c}{\bfseries ACM} \\
\cmidrule(lr){2-3} \cmidrule(lr){4-5} \cmidrule(lr){6-7}
Model $\backslash$ Metric & Macro-F1 & Micro-F1 & Macro-F1 & Micro-F1 & Macro-F1 & Micro-F1 \\
\midrule
HGNN-LE              & \Data{91.19}{0.78} & \Data{91.84}{0.64} & \Data{47.58}{4.18} & \Data{57.80}{2.65} & \Data{93.67}{0.81} & \Data{93.55}{0.78} \\
SeHGNN               & \Data{93.62}{0.20} & \Data{94.06}{0.20} & \Data{65.67}{0.49} & \Data{68.40}{0.40} & \Data{93.11}{0.53} & \Data{93.20}{0.51} \\
HINormer+            & \Data{93.94}{0.20} & \Data{94.43}{0.26} & \Data{65.52}{1.96} & \Data{68.33}{1.44} & \Data{94.01}{0.76} & \Data{93.97}{0.93} \\
DiffGraph+           & \Data{92.54}{0.93} & \Data{92.79}{0.84} & \Data{59.57}{1.54} & \Data{64.57}{1.14} & \Data{94.18}{0.87} & \Data{94.18}{0.98} \\
HeterSEED-NM & \Data{91.87}{0.47} & \Data{92.43}{0.40} & \Data{48.99}{4.29} & \Data{58.63}{2.79} & \Data{93.91}{1.44} & \Data{93.72}{1.40} \\
\midrule
\rowcolor{spiderblue} 
\textbf{HeterSEED (Full)} & \BestData{94.48}{0.81} & \BestData{94.87}{0.73} & \BestData{66.71}{0.73} & \BestData{70.60}{0.82} & \BestData{94.34}{0.81} & \BestData{94.21}{0.77} \\
\bottomrule
\end{tabular}
\end{table*}

\begin{table}[h]
\centering
\vspace{-2mm}
\caption{Disentanglement of label injection vs. structural modeling on large-scale datasets.}
\label{tab:label_injection_large}
\setlength{\tabcolsep}{6pt} {
\footnotesize
\begin{tabular}{l|cc|cc}
\toprule
\textbf{Dataset} & \multicolumn{2}{c|}{\bfseries MAG} & \multicolumn{2}{c}{\bfseries RCDD} \\
\cmidrule(lr){2-3} \cmidrule(lr){4-5}
Model $\backslash$ Metric & Macro-F1 & Micro-F1 & Macro-F1 & Micro-F1 \\
\midrule
HGNN-LE                 & \Data{9.34}{0.32} & \Data{29.96}{1.02} & \Data{92.02}{0.77} & \Data{98.15}{0.20} \\
HeterSEED-NM    & \Data{9.83}{0.41} & \Data{30.85}{0.97} & \Data{92.42}{0.26} & \Data{98.44}{0.07} \\
\midrule
\rowcolor{spiderblue} 
\textbf{HeterSEED (Full)} & \BestData{34.99}{0.22} & \BestData{56.25}{0.25} & \BestData{93.09}{0.15} & \BestData{98.60}{0.03} \\
\bottomrule
\end{tabular}
}
\end{table}
Second, we benchmark against \textbf{propagation-style heterogeneous baselines}, including SeHGNN, HINormer+, and DiffGraph+. Notably, their ``+'' variants are further equipped with our label-guided masking strategy to ensure a fair comparison. Despite this enhanced supervision interface, these models consistently underperform HeterSEED. This indicates that neither label-aware propagation nor architectural variations on conventional message passing can match the effectiveness of our structural decoupling approach.

Third, \textbf{HeterSEED-NM} disables the masking mechanism while keeping the rest of the model architectures intact. Interestingly, the empirical impact of the masking mechanism varies significantly depending on the structural properties of the datasets. On relatively homophilic or structurally stable graphs (e.g., ACM, DBLP), the performance degradation is moderate. However, on highly heterophilic datasets like IMDB, removing the masking mechanism leads to a drastic drop in performance (e.g., dropping by nearly $17\%$ in Macro-F1). This stark contrast reveals a critical insight: when node features and labels are severely misaligned, naive label injection suffers from severe confirmation bias. In such scenarios, the masking mechanism acts as an indispensable safeguard rather than a mere regularizer. It forces the model to genuinely learn and exploit the decoupled structural patterns, preventing it from taking a ``shortcut'' by simply memorizing the explicitly leaked label information.

Finally, by comparing the full model with the variants in our core ablation study (Table~\ref{app_tab:ablation}) that remove the pseudo-label-guided structural components (e.g.,  w/o Pseudo-label, w/o Dec, w/o Homo, w/o Hetero, w/o Mask), we observe a substantial performance degradation, especially on heterophilic datasets. This confirms that the heterophily-aware structural channel is the cornerstone for capturing informative neighborhood patterns.

Taken together, these controlled experiments provide a clear and definitive attribution of our model's performance: the gains of HeterSEED do not merely stem from injecting or propagating label information. While label embeddings offer auxiliary feature augmentation, it is the pseudo-label-guided structural modeling that plays the decisive role. Specifically, pseudo-labels serve strictly as structural cues to partition neighbors into homophilic and heterophilic groups, thereby enabling effective heterophily-aware aggregation. Importantly, we clarify that these pseudo-labels are not propagated as explicit supervision signals, but are utilized exclusively to guide the structural disambiguation within the structure channel.

\subsection{Performance under Varying Homophily}
\label{app:homophily}
To complement the RQ4 analysis on real data in the main text, we further conduct a robustness stress test by synthetically manipulating homophily in a controlled setting. Concretely, for a target node type (e.g., \emph{author}) and a given meta-path scheme $\mathcal{P}$ (e.g., $A\!\to\!P\!\to\!A$), we apply an \textbf{SBM-based label injection} mechanism to control the Metapath-based Label Homophily (MLH). We introduce a swap ratio $\rho \in [0,1]$ and generate two injection modes: (i) \emph{high-homophily injection}, where with probability $\rho$ all nodes in a meta-path–induced clique share the same (random) label, mimicking strong community coherence; and (ii) \emph{low-homophily injection}, where labels within each clique are shuffled to enforce diversity, simulating adversarial heterophily. By varying $\rho$ and the injection mode, we obtain a family of synthetic benchmarks that span five MLH intervals $(0,0.2], (0.2,0.4], \dots, (0.8,1.0]$.
\vspace*{-2mm}
\begin{figure}[htbp]
    \centering    \includegraphics[width=0.78\textwidth]{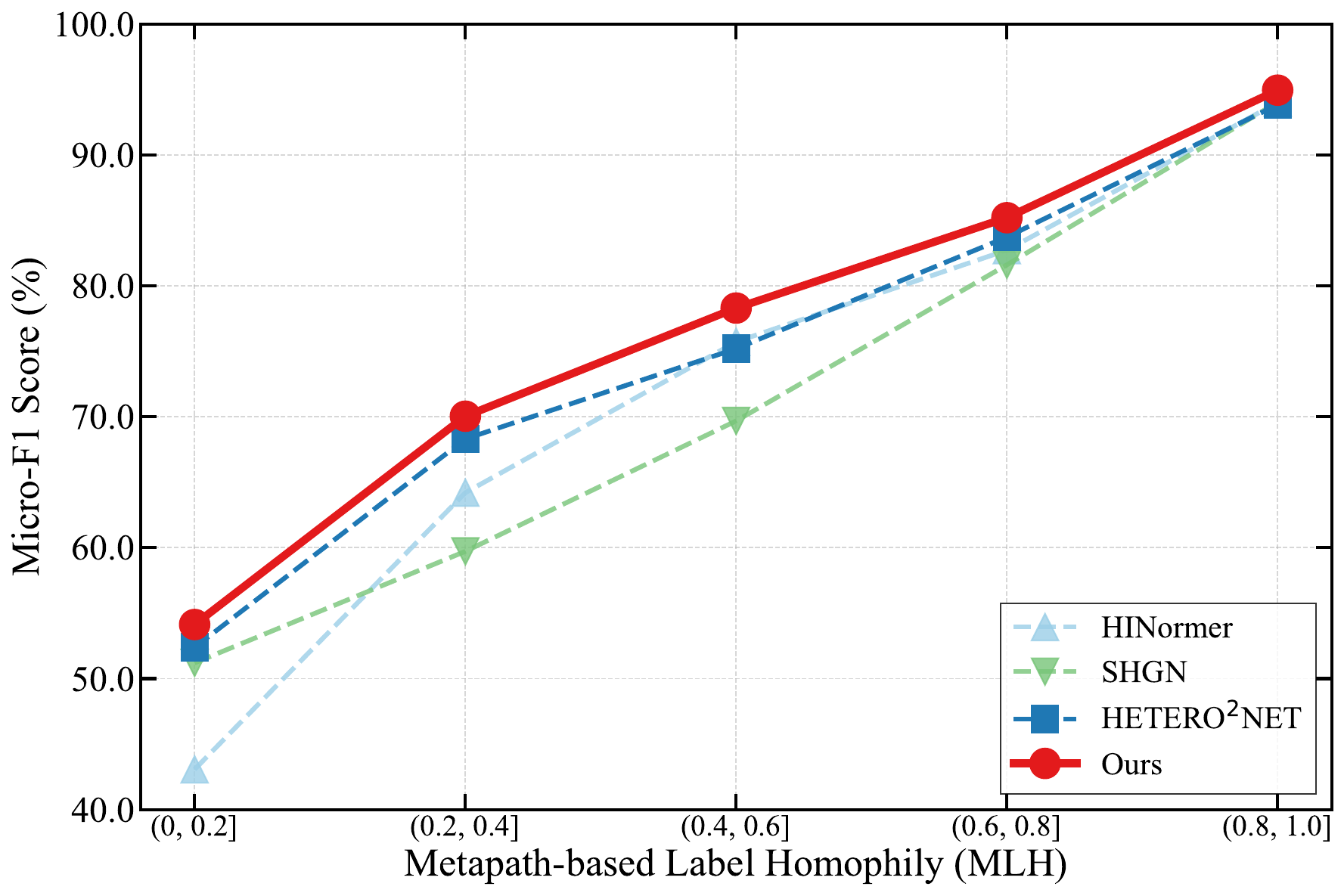} \vspace{-2mm}
    \caption{Micro-F1 of HeterSEED and baselines across Metapath-based Label Homophily (MLH) levels.}    \label{fig:homophily_analysis}\vspace*{-3mm}
\end{figure}

Figure~\ref{fig:homophily_analysis} reports the Micro-F1 scores of HeterSEED and three strong HGNN baselines across these MLH intervals. The performance of all methods generally improves as MLH increases, but HeterSEED consistently achieves the highest accuracy in every interval. In particular, in the low-homophily regime (e.g., $\mbox{MLH} \in (0,0.2]$), where conventional HGNNs are most vulnerable, HeterSEED still maintains a clear margin over the strongest baseline HETERO$^2$NET (54.13\% vs.\ 53.39\%). 

This trend is consistent with our theoretical results in Section~4: Theorem~\ref{thm:expressiveness_main} shows that the semantics–structure decoupling in HeterSEED yields a strictly richer hypothesis class on heterophilic heterogeneous graphs, while Proposition~\ref{prop:bias_main} demonstrates that separating homophilic and heterophilic neighbors reduces the prediction bias induced by heterophilic edges. When MLH is small, these advantages are most pronounced, leading to larger empirical gains; as MLH increases, the bias term diminishes and the performance gap narrows but remains positive. Overall, the synthetic experiments empirically support the theoretical insights that HeterSEED both enhances expressiveness and mitigates heterophily-induced bias.

\subsection{Empirical Validation of Synthetic Benchmarks}
\label{appendix:sanity_check}
To ensure that the synthetic benchmarks used in our robustness analysis behave in a controlled and meaningful way, we perform a sanity check on the DBLP dataset. Figures~\ref{fig:sanity_micro} and~\ref{fig:sanity_macro} report the APA metapath homophily $H$ and the corresponding Micro-F1 / Macro-F1 of HeterSEED under different label perturbation intensities $\rho$.

\textbf{Observation 1 (Controllability of topology).} As the swap ratio $\rho$ increases from $0.0$ to $1.0$, the observed APA homophily $H$ monotonically decreases from about $0.80$ to a stable floor around $0.20$. This confirms that our SBM-based label injection mechanism reliably drives the graph from a homophilous to a heterophilous regime.

\textbf{Observation 2 (Performance sensitivity).} Both Micro-F1 and Macro-F1 degrade smoothly as $H$ decreases, indicating that stronger heterophily indeed makes the prediction task more challenging. Similar trends are observed for baseline HGNNs (omitted for brevity), supporting the view that heterophily is a primary bottleneck for conventional heterogeneous GNNs.

These observations validate that the synthetic MLH-controlled benchmarks used in Section~\ref{rq4_exp} faithfully reflect the impact of varying homophily and provide a sound basis for the interval-based robustness evaluation in the main text.
\begin{figure}[htbp]
    \centering
    \begin{subfigure}[b]{0.48\textwidth}
        \centering
        \includegraphics[width=\linewidth]{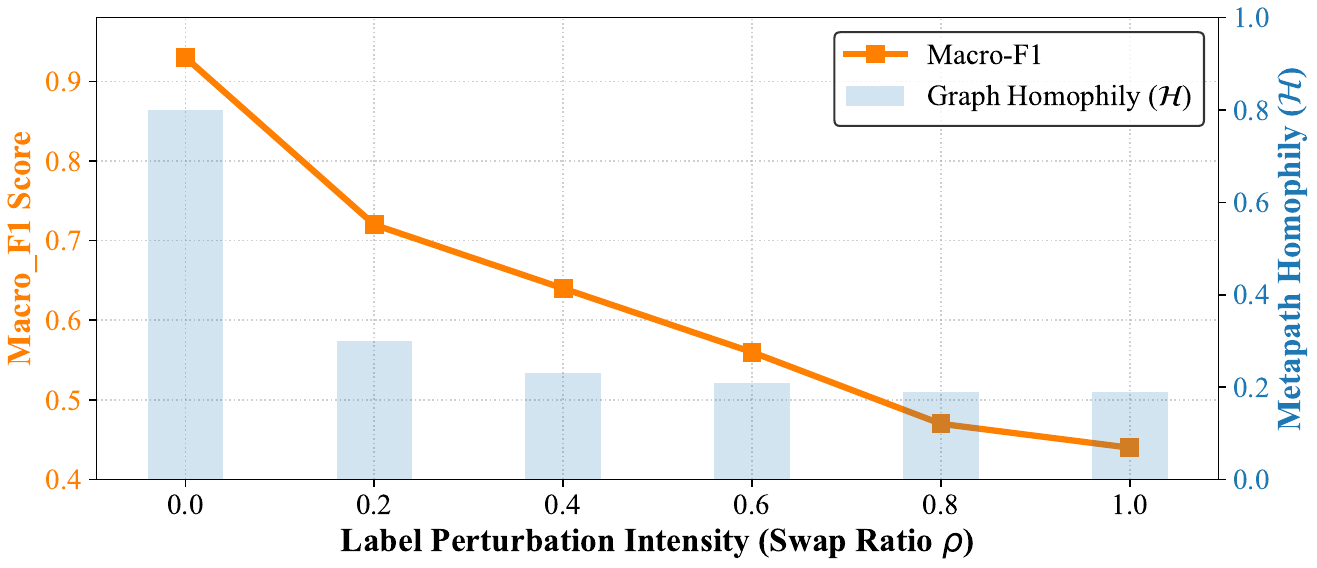}
        \caption{Micro-F1 Performance}
        \label{fig:sanity_micro}
    \end{subfigure}
    \hfill 
    \begin{subfigure}[b]{0.48\textwidth}
        \centering
        \includegraphics[width=\linewidth]{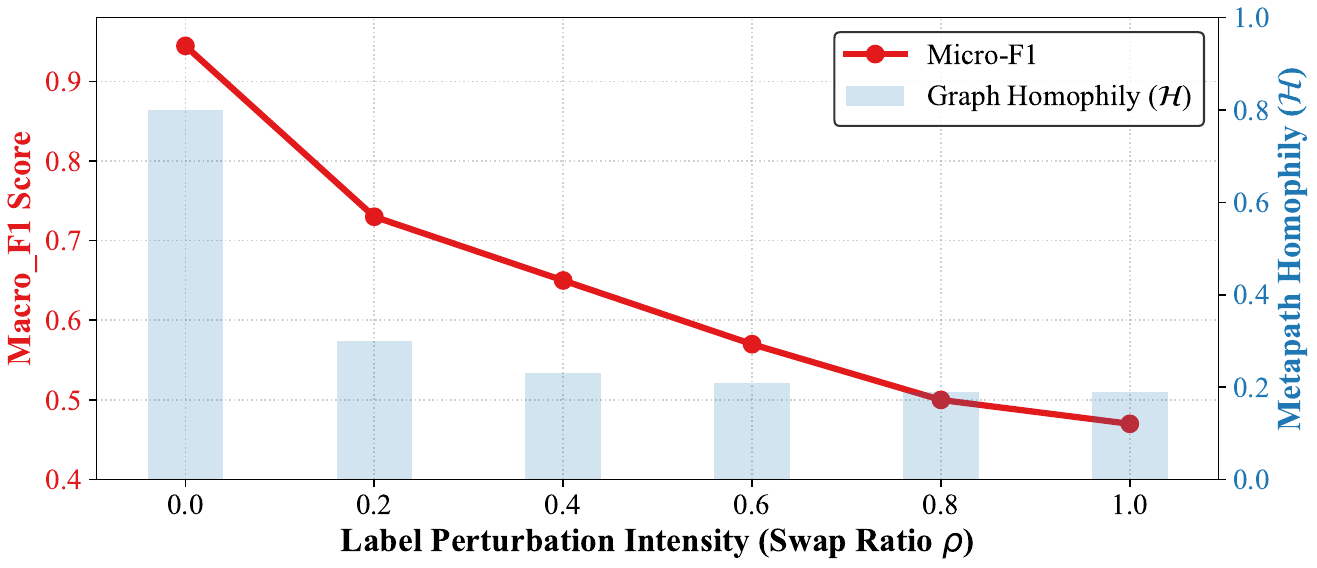}
        \caption{Macro-F1 Performance}
        \label{fig:sanity_macro}
    \end{subfigure}    
    \caption{Effect of label perturbation intensity $\rho$ on APA metapath homophily $H$ and model performance on DBLP. Panels (a) and (b) show Micro-F1 and Macro-F1 of HeterSEED and the corresponding APA homophily $H$, respectively.}
    \label{fig:appendix_sanity_comparison}\vspace{-3mm}
\end{figure}

\subsection{Average Precision on the RCDD Dataset}
\label{app:rcdd_ap}

Figure~\ref{fig:RCDD_AP} reports the Average Precision (AP) of all models on the large-scale RCDD dataset. \textbf{HeterSEED} attains the highest AP and thus ranks first, while HETERO$^{2}$NET is consistently the second-best method; the remaining models (SHGN, HGT, RGAT, HAN, RGCN, ACM-GCN, FAGCN, GAT, GCN) lag further behind. In particular, the margin between HeterSEED and the strongest baseline is noticeable, and HeterSEED also exhibits small variance across runs, indicating stable performance. Together with the F1 results in the main text, this AP comparison suggests that the proposed semantics–structure decoupling and heterophily-aware aggregation enable HeterSEED to cope with the severe heterophily and class imbalance of RCDD more effectively than existing HGNNs and heterophily-aware baselines.
 \vspace{-4mm}
\begin{figure}[htbp!]
    \centering
    \includegraphics[width=0.98\linewidth]{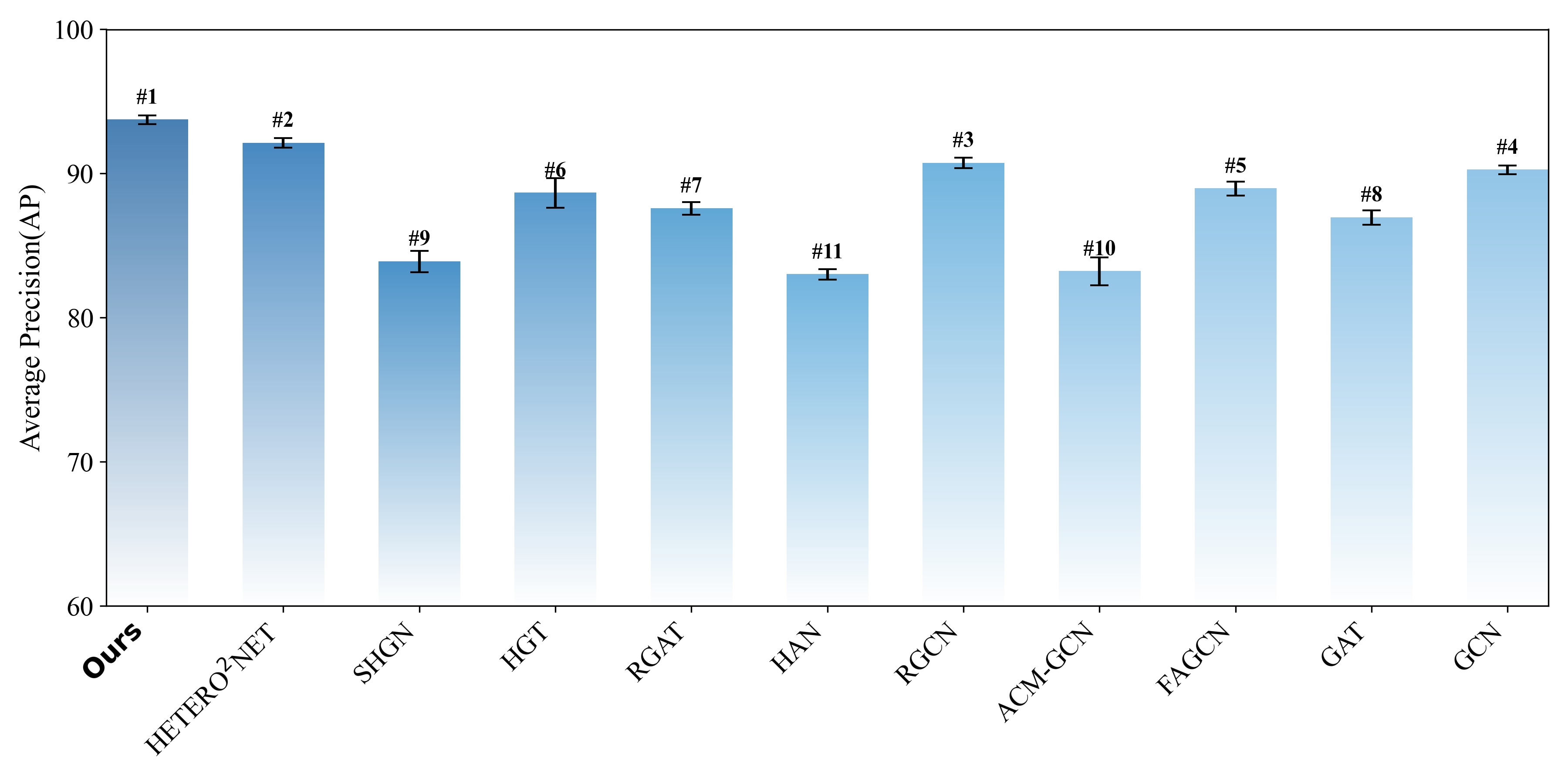} 
    \vspace{-3mm}
    \caption{Average Precision (AP) of different models on the RCDD dataset.}
    \label{fig:RCDD_AP}
\end{figure}

\subsection{Additional Results of Hyperparameter Sensitivity Analysis}
\label{app:add results}
\begin{figure*}[htbp!]
    \centering
    \begin{subfigure}[b]{0.45\textwidth}
        \centering
        \includegraphics[width=\textwidth]{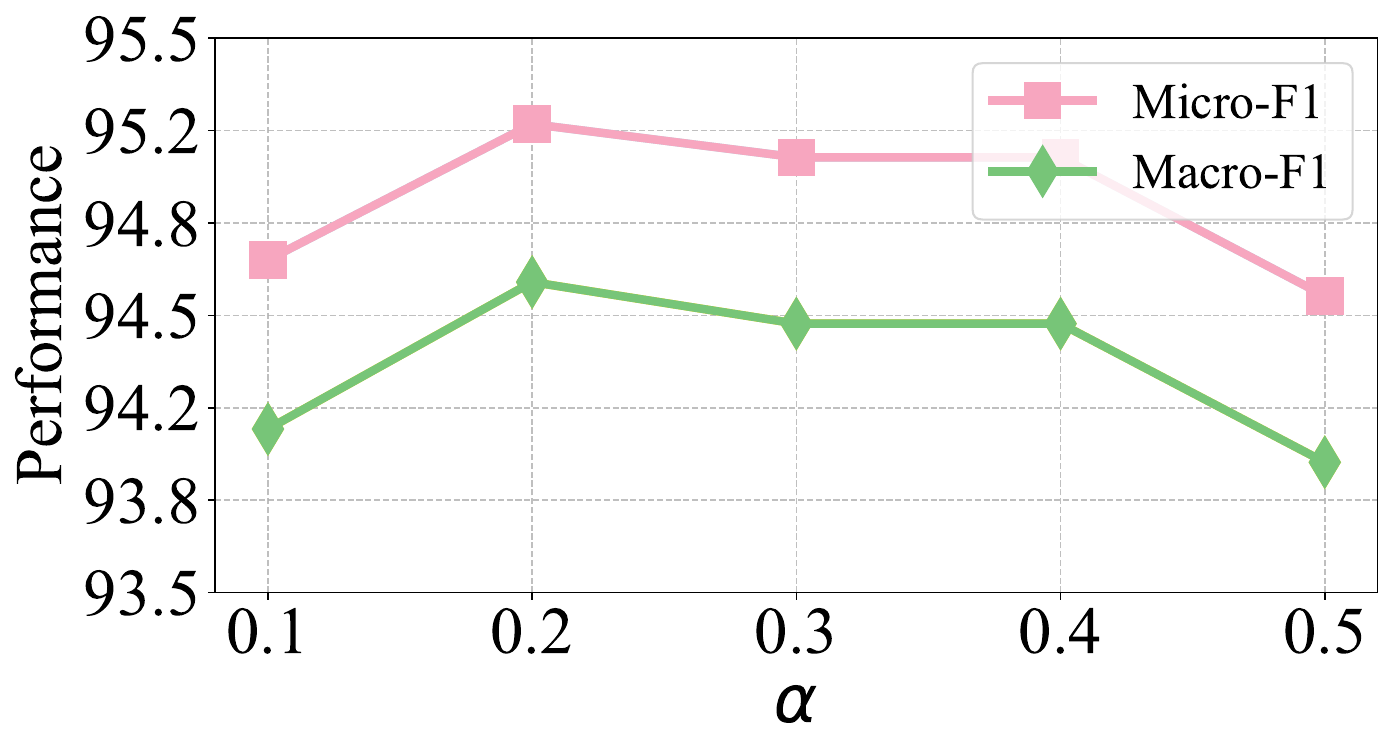}
        \label{fig:dblp_alpha}
    \end{subfigure}
    \hfill
    \begin{subfigure}[b]{0.45\textwidth}
        \centering
        \includegraphics[width=\textwidth]{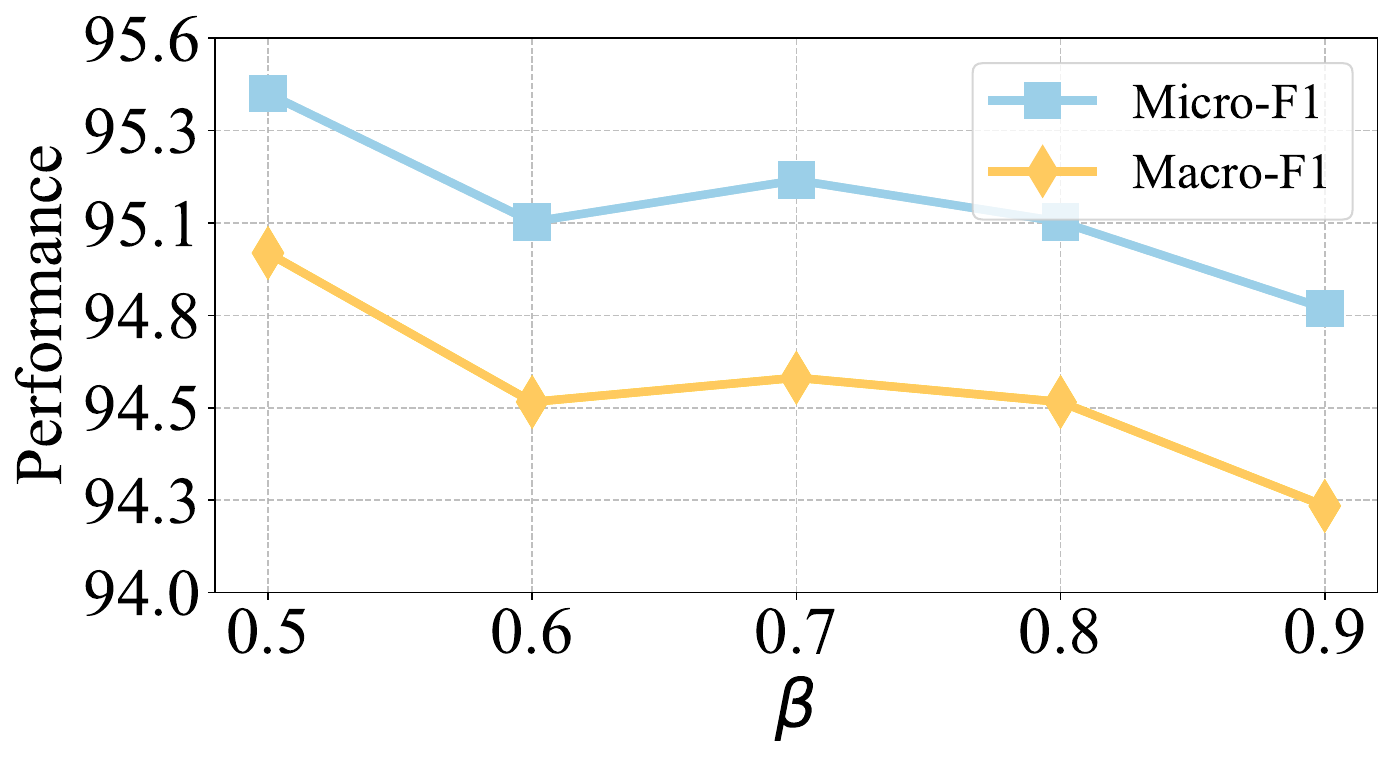}
        \label{fig:dblp_beta}
    \end{subfigure}
    \vspace{2mm} 

    \begin{subfigure}[b]{0.45\textwidth}
        \centering
        \includegraphics[width=\textwidth]{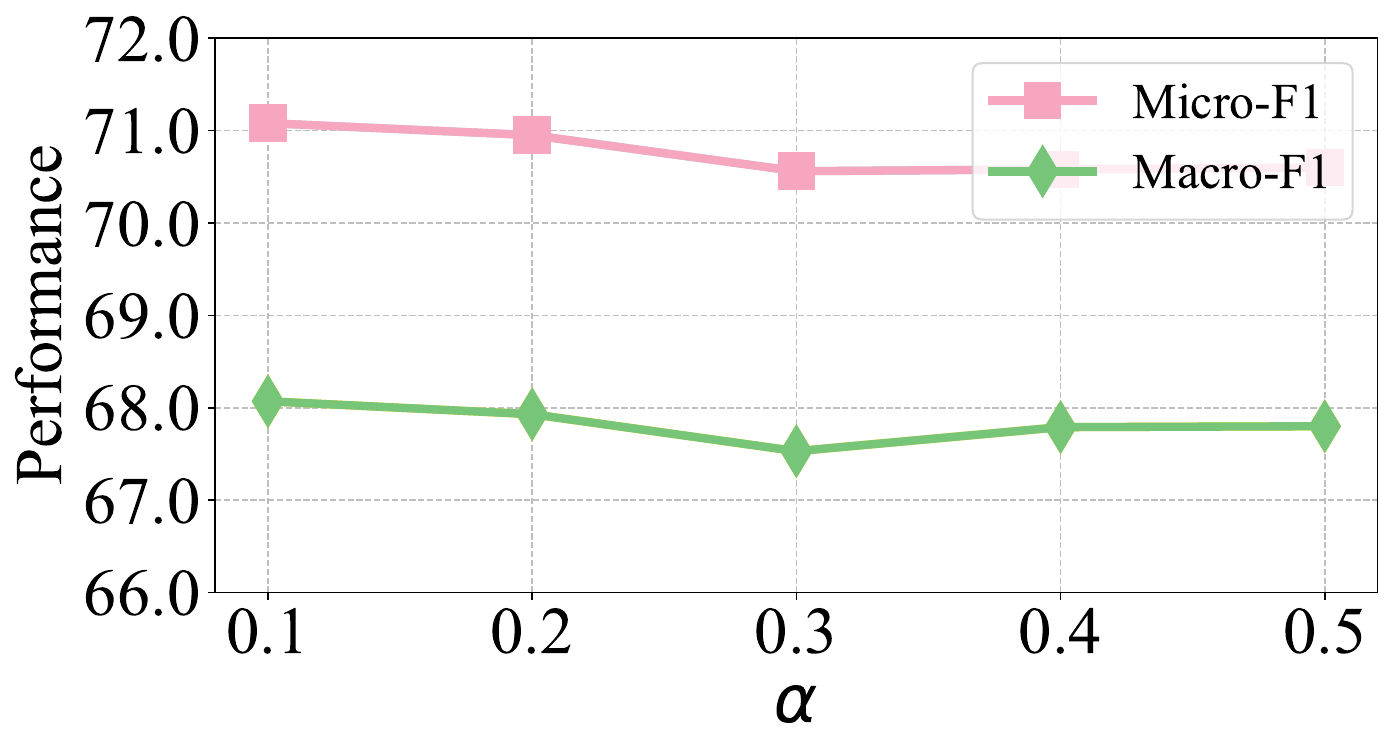}
        \label{fig:imdb_alpha}
    \end{subfigure}
    \hfill
    \begin{subfigure}[b]{0.45\textwidth}
        \centering
        \includegraphics[width=\textwidth]{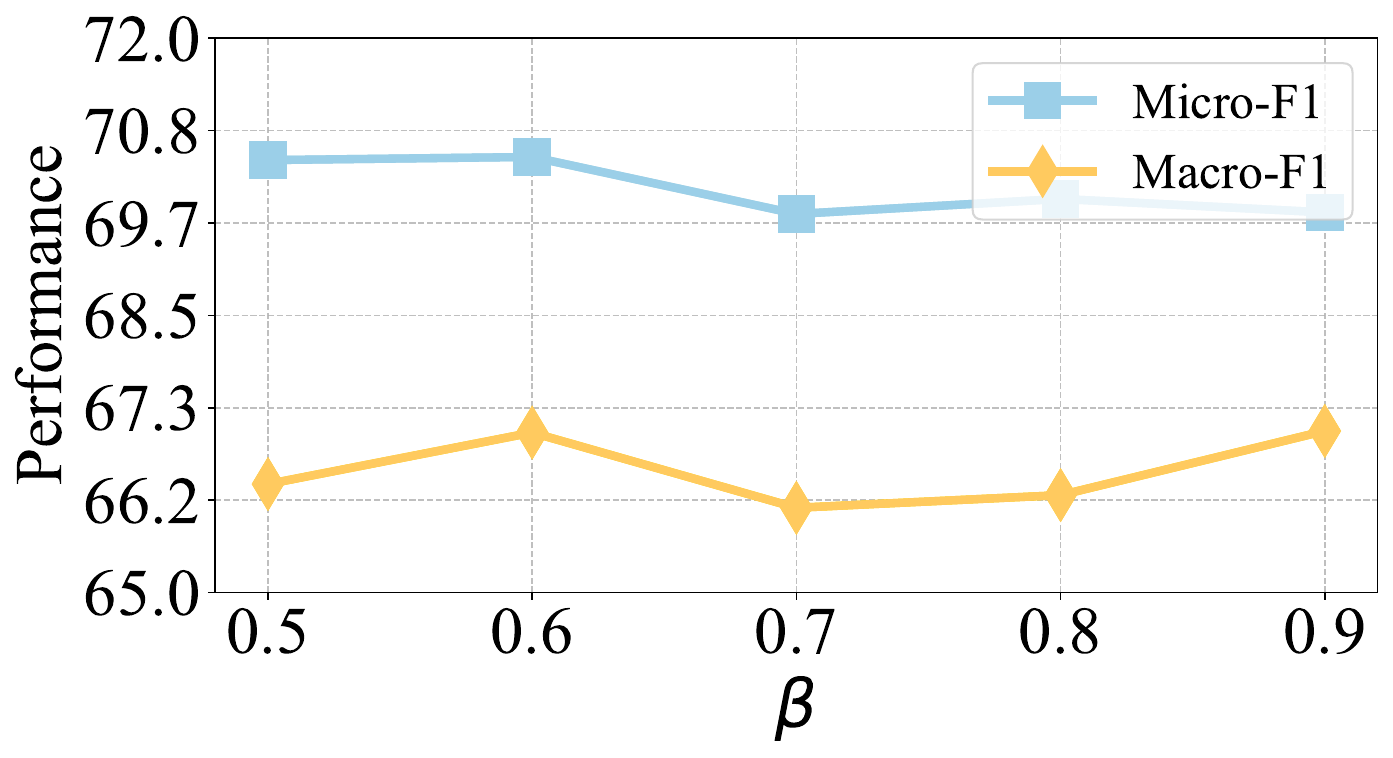}
        \label{fig:imdb_beta}
    \end{subfigure}
    \vspace{2mm}

    \begin{subfigure}[b]{0.45\textwidth}
        \centering
        \includegraphics[width=\textwidth]{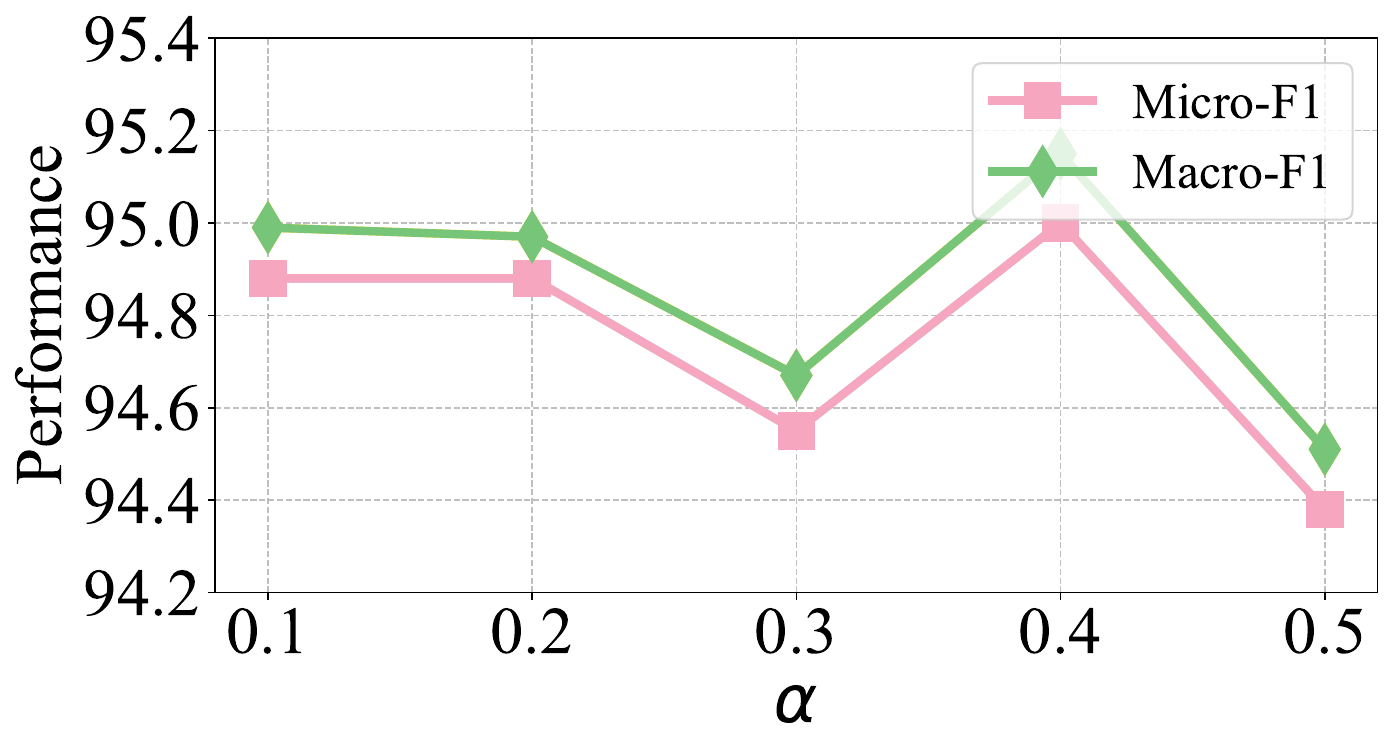}
        \label{fig:acm_alpha}
    \end{subfigure}
    \hfill
    \begin{subfigure}[b]{0.45\textwidth}
        \centering
        \includegraphics[width=\textwidth]{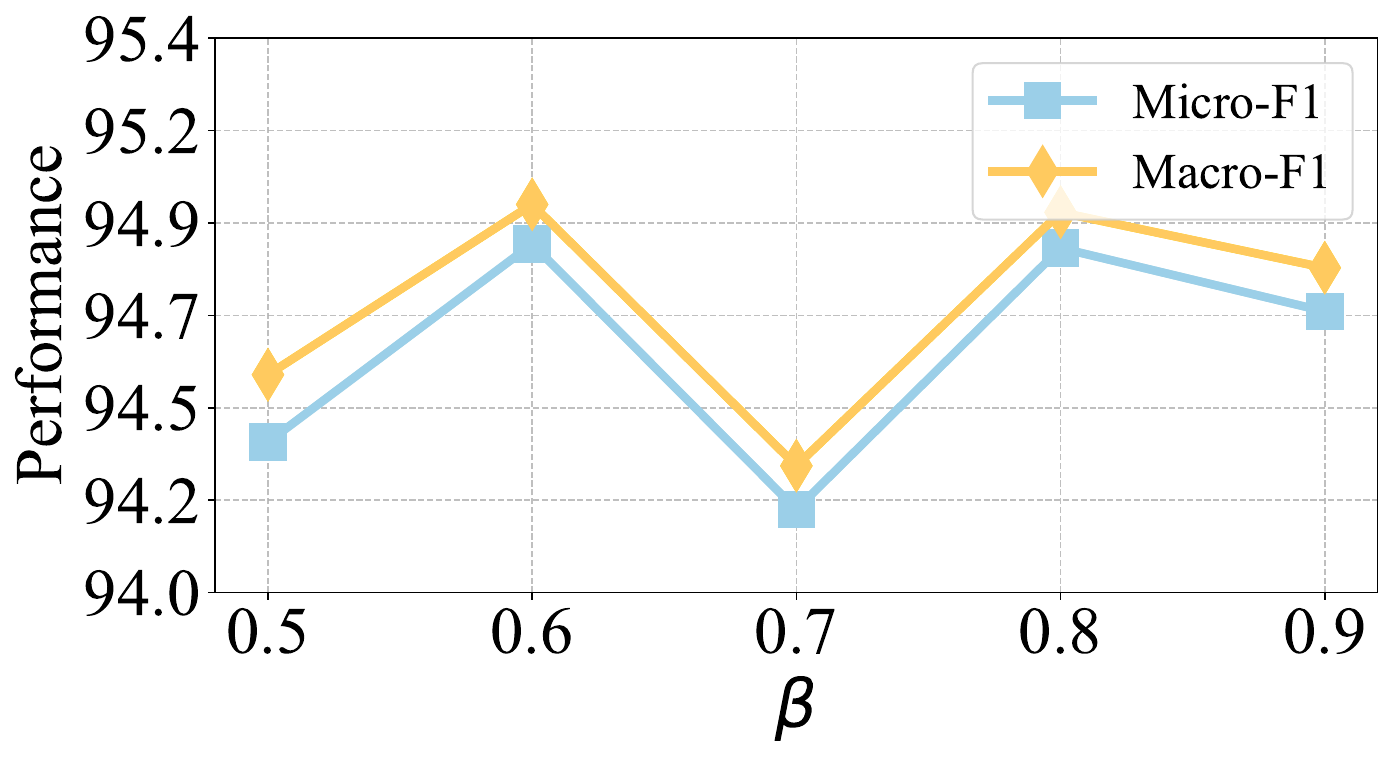}
        \label{fig:acm_beta}
    \end{subfigure}
    
    \caption{Hyperparameter sensitivity of HeterSEED on DBLP (top), IMDB (middle), and ACM (bottom). For each dataset, the left plot shows the performance versus the decoupling coefficient $\alpha$, and the right plot shows the effect of the masked label rate $\beta$.}
    \label{fig:sensitivity_benchmark}
\end{figure*}

\begin{figure*}[htbp!]
    \centering

    \begin{subfigure}[b]{0.45\textwidth}
        \centering
        \includegraphics[width=\textwidth]{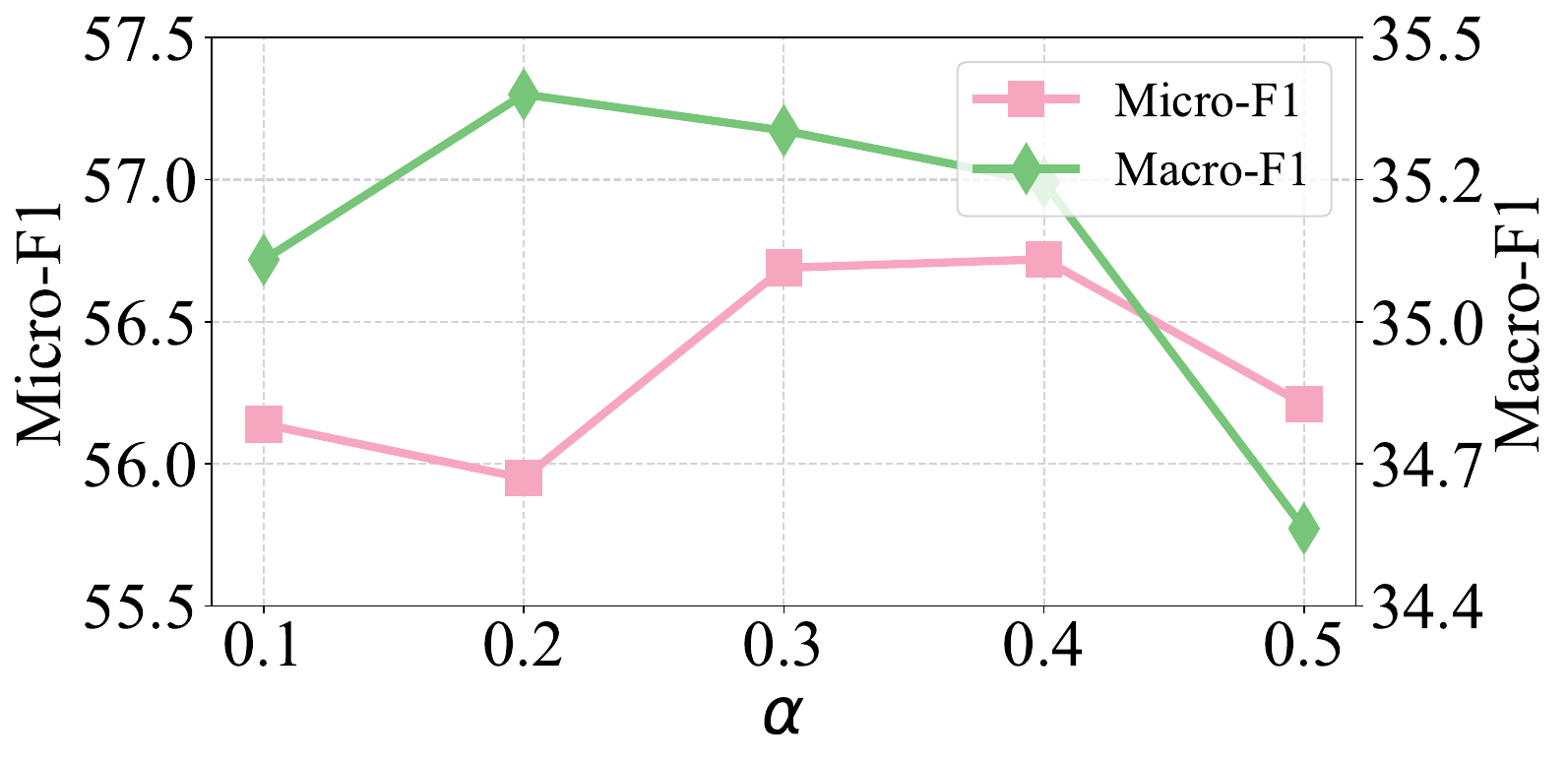}
        \label{fig:mag_alpha}
    \end{subfigure}
    \hspace{0.05\textwidth}
    \begin{subfigure}[b]{0.45\textwidth}
        \centering
        \includegraphics[width=\textwidth]{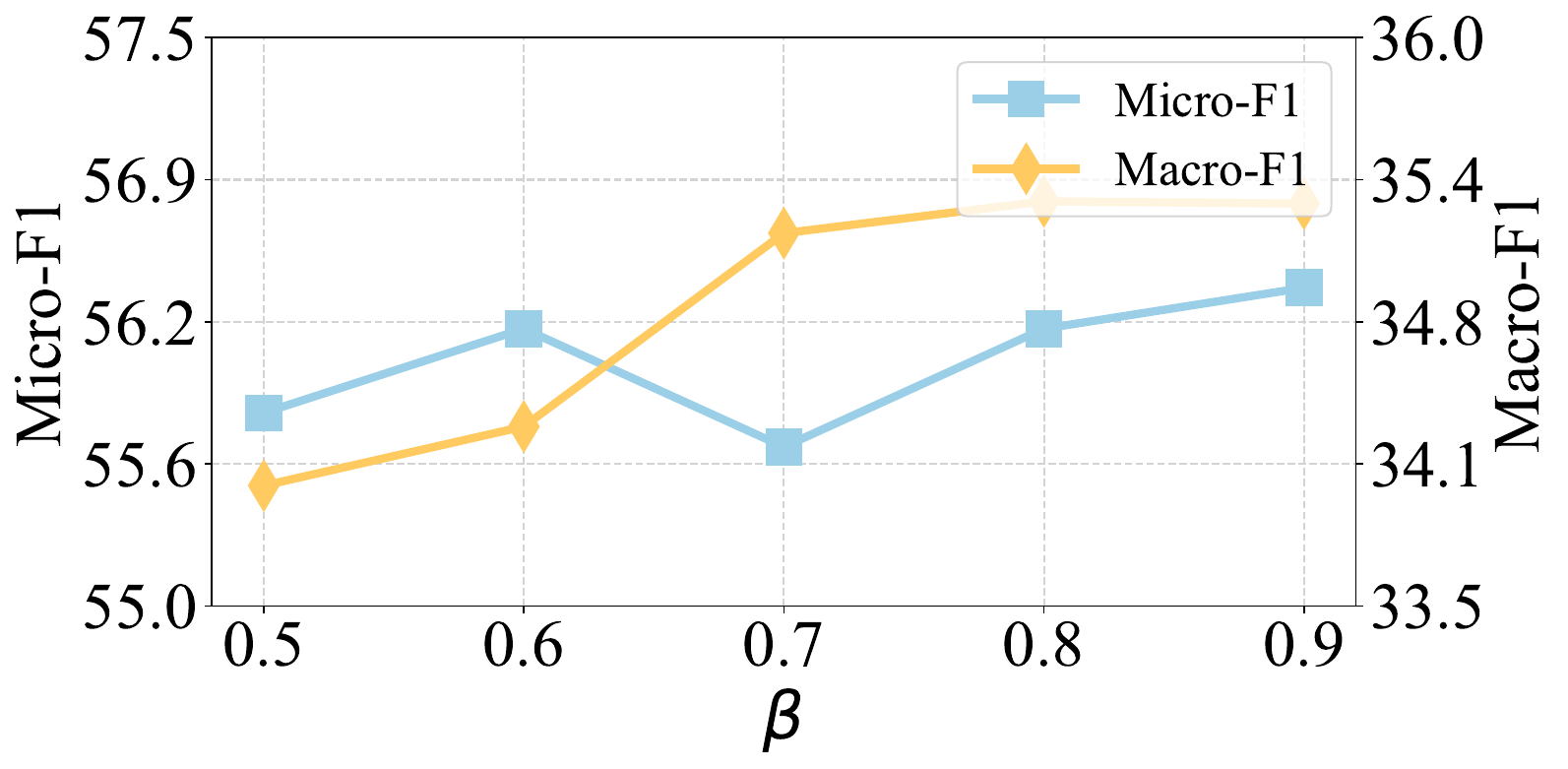}
        \label{fig:mag_beta}
    \end{subfigure}
\vspace{-2mm}
    \begin{subfigure}[b]{0.45\textwidth}
        \centering
        \includegraphics[width=\textwidth]{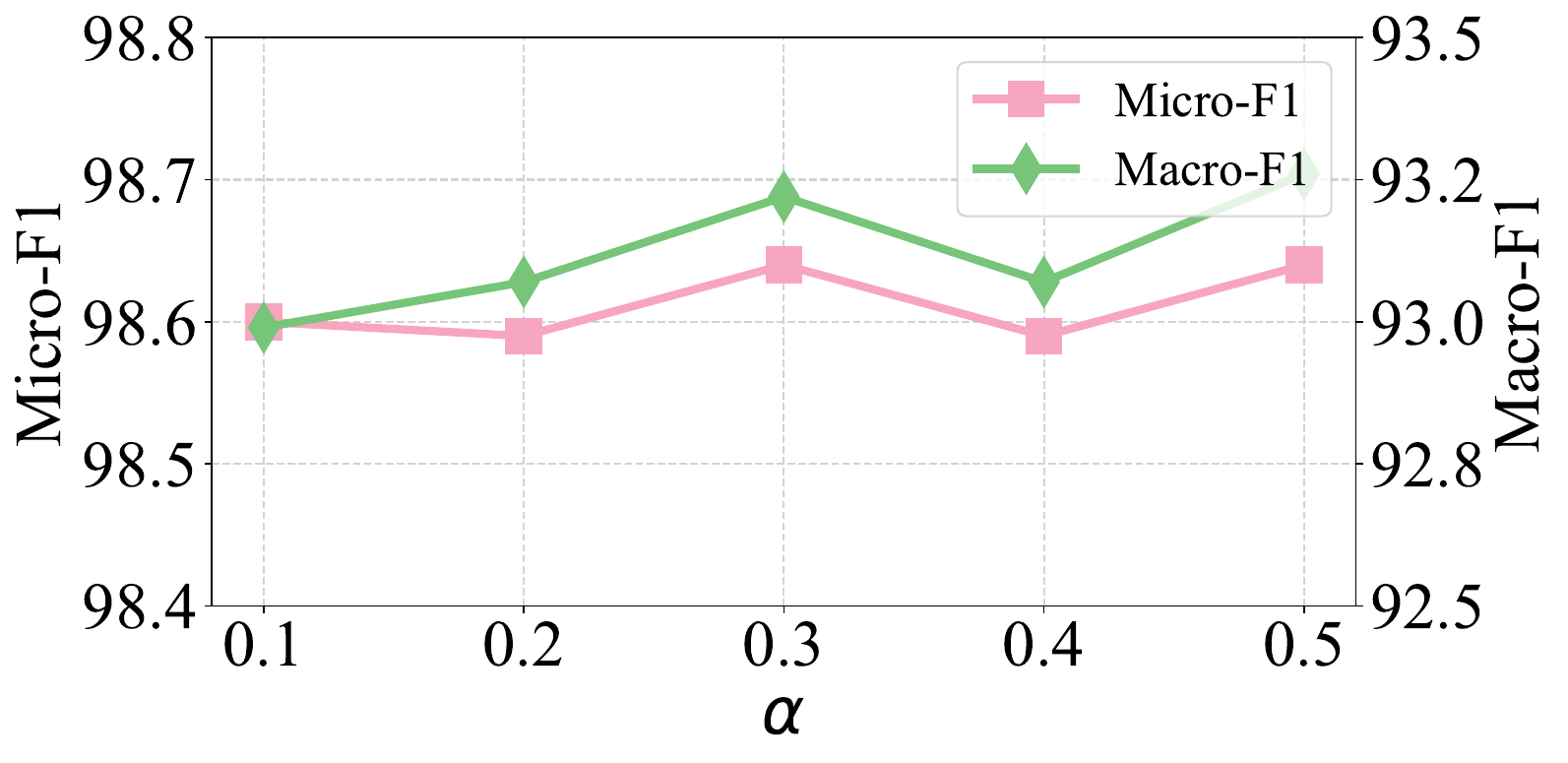}
        \label{fig:rcdd_alpha}
    \end{subfigure}
    \hspace{0.05\textwidth}
    \begin{subfigure}[b]{0.45\textwidth}
        \centering
        \includegraphics[width=\textwidth]{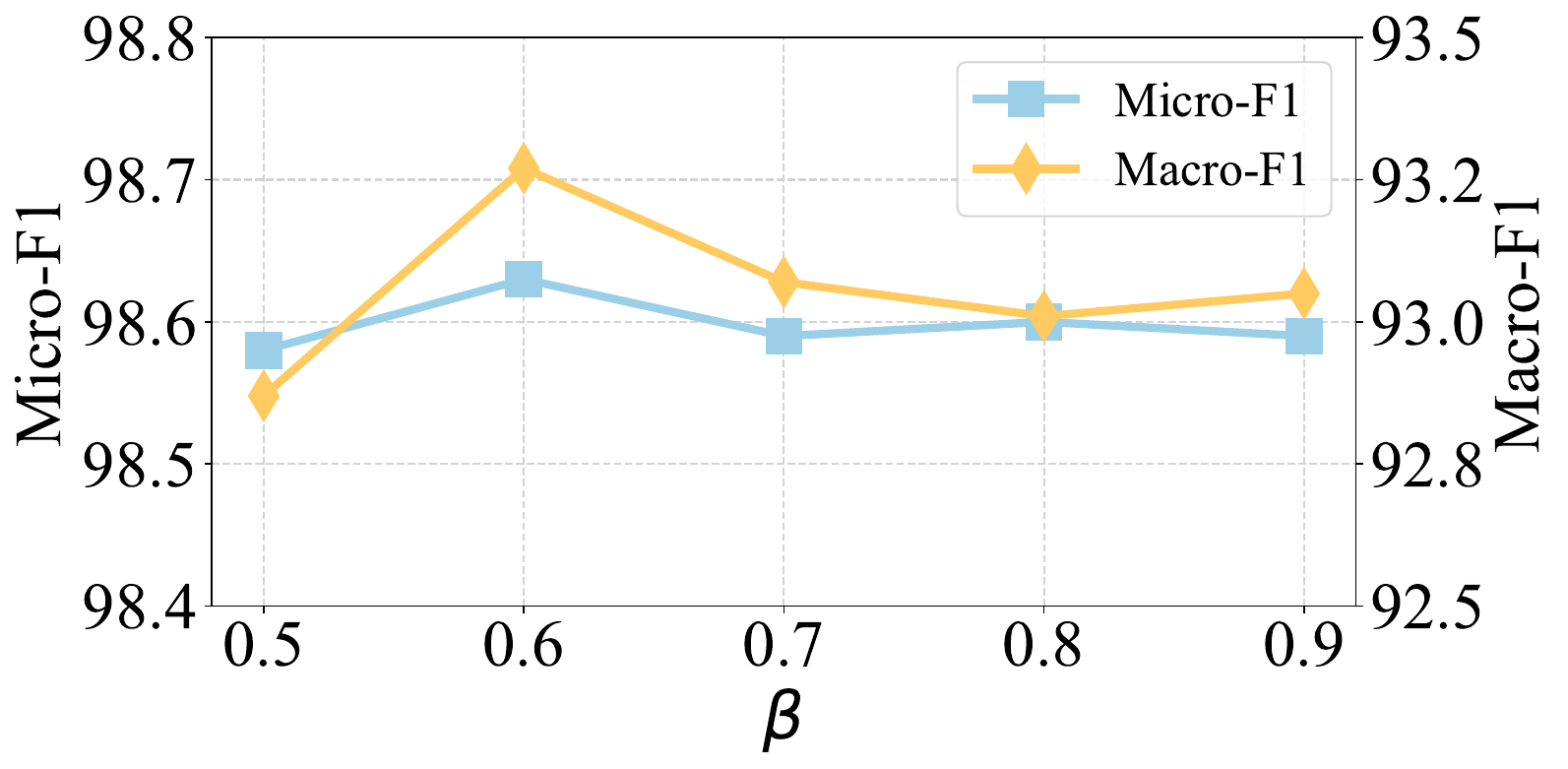}
        \label{fig:rcdd_beta}
    \end{subfigure}
    \vspace{-4mm}
    \caption{Hyperparameter sensitivity of HeterSEED on the large-scale datasets: MAG (top) and RCDD (bottom). The left plot reports Micro-F1 and Macro-F1 versus the decoupling coefficient $\alpha$, and the right plot reports the effect of the masked label rate $\beta$.}
    \label{fig:sensitivity_large}
\end{figure*}

We systematically study hyperparameter sensitivity of HeterSEED on three representative heterogeneous datasets: DBLP, IMDB, and ACM, as shown in Figure~\ref{fig:sensitivity_benchmark}. In addition, we evaluate large-scale datasets in Figure~\ref{fig:sensitivity_large}. It is worth noting that on large-scale datasets, the value ranges of Micro-F1 and Macro-F1 differ significantly; therefore, Figure~\ref{fig:sensitivity_large} adopts dual y-axes to separately present the two metrics. In contrast, since both metrics fall within similar ranges on the benchmark datasets, Figure~\ref{fig:sensitivity_benchmark} uses a single shared y-axis for clarity. 

For each dataset, we analyze the effect of two key hyperparameters: the decoupling coefficient $\alpha$ in the objective function $\mathcal{L}_{\text{cls}} + \alpha \mathcal{L}_{\text{dec}}$, and the masked-label rate $\beta$ in the label-guided masking mechanism. Specifically, the left subfigures report the performance variation with respect to $\alpha$, while the right subfigures show the effect of $\beta$.

Across all datasets, HeterSEED exhibits low sensitivity to both $\alpha$ and $\beta$: performance curves remain smooth, and only minor fluctuations appear under extreme settings. These results are consistent with the observations in the main text and further indicate that HeterSEED is robust and reliable across heterogeneous graphs of different scales and structural characteristics.\vspace{-2mm}

\section{Limitations and Broader Impacts}
\label{app:limitations}

\noindent\textbf{Limitations.} HeterSEED is currently developed for semi-supervised node classification on heterogeneous graphs under heterophily, and its present formulation is most natural when a small set of informative symmetric metapaths is available. In particular, the structure-aware channel uses pseudo-labels to separate homophilic and heterophilic neighborhoods. Although the iterative pseudo-label updating strategy makes the framework reasonably robust in practice, its reliability can further benefit from confidence-aware pseudo-label refinement in extremely low-label or highly noisy settings. In addition, the current implementation assumes manually specified symmetric metapaths, which are standard and effective on the benchmarks considered here, but automatic metapath discovery or selection may further broaden applicability to domains with richer or less well-understood schemas. These considerations primarily delineate the current scope of the method and suggest natural directions for future work, including confidence-aware pseudo-label refinement, adaptive metapath learning, and extensions of semantics--structure decoupling to heterogeneous hypergraphs under heterophily.

\noindent\textbf{Broader Impacts.} The proposed HeterSEED framework advances heterogeneous graph representation learning by improving robustness to heterophily through explicit semantics--structure decoupling. This can benefit a range of applications, such as recommendation, social network analysis, and knowledge graph mining, where heterogeneous relations often involve both homophilic and heterophilic patterns. Beyond predictive performance, the explicit separation of semantic and structural channels, together with node-level adaptive fusion, may also provide a more transparent view of whether a prediction is supported primarily by local semantics or by higher-order structural evidence. As with other general-purpose graph learning methods, deployment in real systems should remain attentive to standard considerations such as data quality, sampling bias, privacy, fairness, and appropriate human oversight in high-impact settings. Overall, we view HeterSEED as a methodological contribution intended to broaden the modeling toolkit for heterogeneous graphs rather than a system tailored to any specific consequential decision-making pipeline.




\newpage

\end{document}